\documentclass{article}

\usepackage{microtype}
\usepackage{graphicx}
\usepackage{subcaption}
\usepackage{booktabs} 

\usepackage{hyperref}
\usepackage{xurl}        
\hypersetup{breaklinks=true}

\usepackage[preprint]{icml2026}
\usepackage{amsmath}
\usepackage{amssymb}
\usepackage{mathtools}
\usepackage{amsthm}
\usepackage{tikz}
\usepackage{booktabs}
\usepackage{subcaption}
\usepackage[capitalize,noabbrev]{cleveref}

\theoremstyle{plain}
\newtheorem{theorem}{Theorem}[section]

\newtheorem{lemma}[theorem]{Lemma}

\theoremstyle{definition}
\newtheorem{definition}[theorem]{Definition}
\newtheorem{assumption}[theorem]{Assumption}
\theoremstyle{remark}

\usepackage[textsize=tiny]{todonotes}

\icmltitlerunning{From Ambiguity to Action: A POMDP Perspective on Partial Multi-Label Ambiguity and Its Horizon-One Resolution}

\begin{document}

\twocolumn[
  \icmltitle{From Ambiguity to Action: A POMDP Perspective on Partial Multi-Label Ambiguity and Its Horizon-One Resolution}



  \icmlsetsymbol{equal}{*}

  \begin{icmlauthorlist}
    \icmlauthor{Hanlin Pan}{yyy,sch}
    \icmlauthor{Yuhao Tang}{yyy,sch}
    \icmlauthor{Wanfu Gao}{equal,yyy,sch}

  \end{icmlauthorlist}

  \icmlaffiliation{yyy}{College of Computer Science and Technology, Jilin University, China}
 
  \icmlaffiliation{sch}{Key Laboratory of Symbolic Computation and Knowledge Engineering of Ministry of Education, Jilin University, China}

  \icmlcorrespondingauthor{Wanfu Gao}{gaowf@jlu.edu.cn}

  \icmlkeywords{Machine Learning, ICML}

  \vskip 0.3in
]



\printAffiliationsAndNotice{}  

\begin{abstract}
In partial multi-label learning (PML), the true labels are unobserved, which makes label disambiguation  important but difficult. A key challenge is that ambiguous candidate labels can propagate errors into downstream tasks such as feature engineering. To solve this issue, we jointly model the disambiguation and feature selection tasks as Partially Observable Markov Decision Processes (POMDP) to turn PML risk minimization into expected-return maximization. Stage 1 trains a transformer policy via reinforcement learning to produce high-quality hard pseudo-labels; Stage 2 describes feature selection as a sequential reinforcement learning problem, selecting features step by step and outputting an interpretable global ranking. We further provide the theoretical analysis of PML-POMDP correspondence and the excess-risk bound that decompose the error into pseudo label quality term and sample size. Experiments in multiple metrics and data sets verify the advantages of the framework\footnote{The code is available at \url{https://anonymous.4open.science/r/pomdp-hard-E050}}.
\end{abstract}
\section{Introduction}

Partial multi-label learning (PML) is a weakly supervised learning paradigm \cite{xie2018partial,yang2024noisy}, which often appears in practical fields with ambiguous but not completely missing annotations, such as  noisy topic tags \cite{chen2022representation,murshed2022enhancing}, image annotations  \cite{mazzamuto2022weakly,tejero2023full}, and bioinformatics tasks \cite{mongardi2024biologically,smarandache2024soft}. In PML, a candidate label set is provided for each instance, which is guaranteed to contain at least one unknown ground truth label as Figure \ref{mou}. Therefore, the key is label disambiguation: deciding which labels should be positive \cite{li2025calibrated,hang2023partial,zhong2024negative}. This step is critical since incorrect disambiguation will  mislead downstream tasks, especially in high-dimensional environments \cite{pan2025reconsidering,wang2022partial}.

\begin{figure}[H]
    \centering
    \begin{minipage}{0.7\linewidth}
        \includegraphics[width=\textwidth]{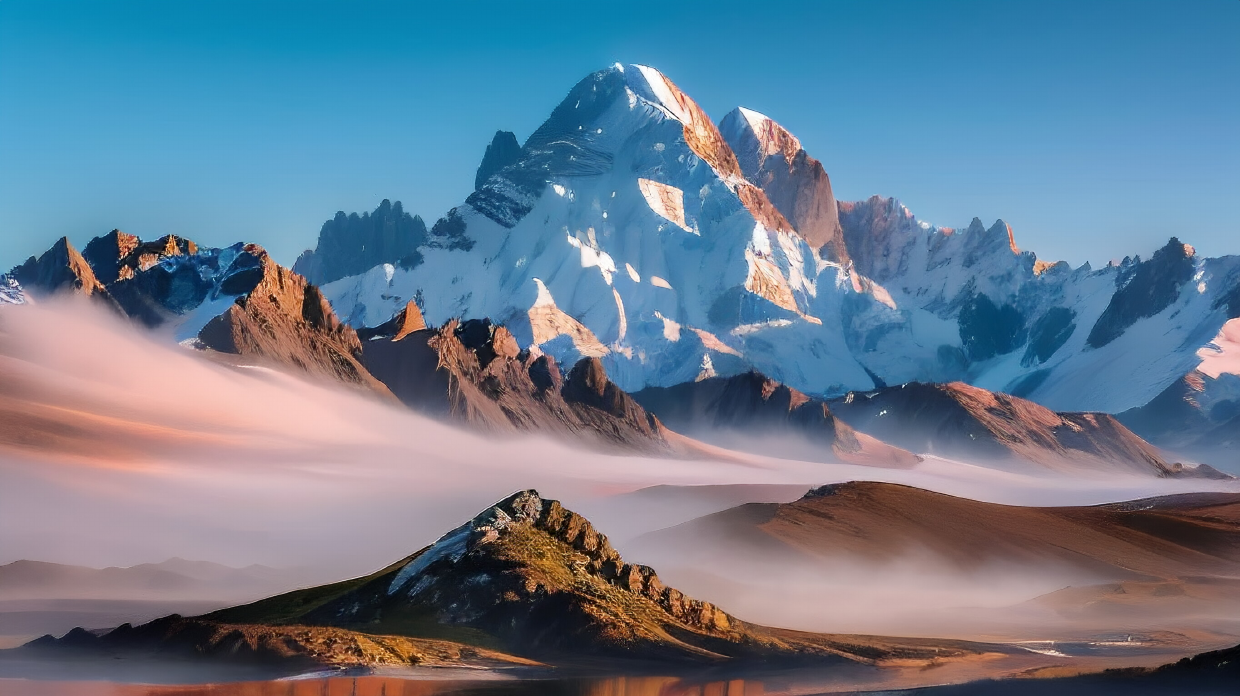}
    \end{minipage}%
    \begin{minipage}{0.3\linewidth}
        \raggedright
        \textbf{Candidate labels}
        \begin{itemize}
            \item \textcolor{red}{sky}
            \item people
            \item \textcolor{red}{mountain}
            \item sun
            \item \textcolor{red}{snow}        
        \end{itemize}
    \end{minipage}
    \caption{An example of PML, only three of five candidate labels are valid (in red).}
    \label{mou}
\end{figure}
\begin{figure*}[t]
\centering
\includegraphics[width=0.9\linewidth]{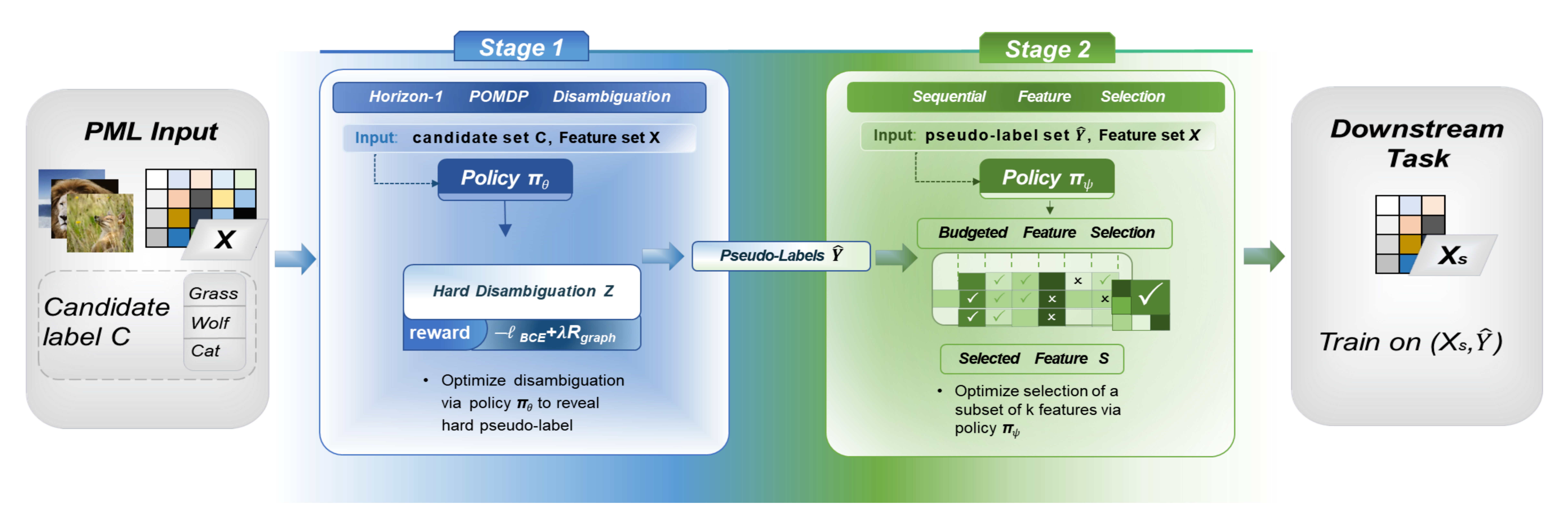}

\caption{Overview of our POMDP framework for PML. Stage 1 transforms the label disambiguation problem under candidate constraints into a horizon-1 decision problem and generates hard pseudo labels. In Stage 2, these pseudo labels are used to learn the budget feature selection strategy, generating a feature subset for downstream task. }
\label{flow}
\end{figure*}
Most of the existing methods either convert candidate labels into soft pseudo labels with logits \cite{zhang2020partial,xu2020partial}, or jointly optimize candidate labels in prediction training \cite{xie2021partial,liang2024partial}. However, soft disambiguation methods may be misled by pseudo candidates with higher logit, leading downstream learning to the wrong direction. The joint optimization needs to learn from two possible error sources: (i) imperfect pseudo labels and (ii) limited sample size and hypothesis complexity. This coupling harms stability and obscures how improvements in disambiguation translate to generalization.

By further abstracting the disambiguation task, the PML problem can be regarded as a one-step decision under constraints: given an observation of the dataset, learners must output a feasible disambiguation, in which the positive value only exists in the candidate set, and the real label is still invisible. Intuitively, these characters exactly match the structure of Partially Observable Markov Decision Processes (POMDPs) \cite{cassandra1998survey,arcieri2024pomdp}. Specifically, we convert PML disambiguation to horizon-1 POMDP, in which the hidden state is the ground truth label, the observation is PML dataset, and the action is the disambiguation. By defining the reward as the negative disambiguation loss, maximizing expected return is equivalent to minimizing the PML disambiguation risk, directly aligning policy optimization with the learning objective.

Building on this fit, we propose a two-stage framework as Figure \ref{flow} shows.  Stage 1 learns a horizon-1 disambiguation policy and generate hard pseudo labels using a discriminative loss together with a structural regularizer. In Stage 2, budget feature selection is performed to output interpretable feature subsets  for downstream tasks.

Our analysis further clarifies this logic: We establish the equivalence between PML and POMDP, provide stationary convergence for Stage 1 policy optimization under the standard smoothness assumption, and derive the excess risk decomposition for downstream predictors trained on pseudo labels: one item only depends on the quality of pseudo labels, while the other item only depends on the sample size and the hypothesis complexity. This decomposition illustrates when better disambiguation will (and will not) improve the final generalization.

We summarize main contributions as follows:
\begin{itemize}
  \item Formalizing PML disambiguation as horizon-1 POMDP, and associate the policy with disambiguation rules, under which minimizing PML risk is equivalent to maximizing return.
  \item Proposing a two-stage framework: In the Stage 1, candidate-constraint hard disambiguation is performed. In the Stage 2, feature selection is performed to generate interpretable feature subsets.
  \item Providing the proof of convergence and excess risk decomposition, separate the pseudo label quality from the sample size/complexity effect, and explain how gains transfer across stages.
\end{itemize}

Across multiple PML method and datasets, we achieve better performance on several metrics, and ablations confirm the importance of Stage 1 disambiguation. 
\section{Related Work}
\textbf{PML disambiguation.}
Existing PML methods can be roughly grouped into (i) soft disambiguation approaches that infer candidate-wise confidence scores \cite{durand2019learning,wang2020semi} or pseudo-label distributions \cite{xu2021progressive,gong2021understanding} and train predictors with reweighting or self-training \cite{duarte2021plm,feng2019partial}, and (ii) structure-based approaches that impose instance/label consistency via graphs \cite{lyu2020partial}, correlations \cite{sun2021global}, or low-rank priors \cite{ijcai2025p576}.
However, soft disambiguation methods are prone to assigning positive gradients to false candidates with relatively high scores, which may lead downstream learning (and feature selection) in the wrong direction. Although structure based methods are useful when assuming precise relationships, they are fragile to noise or incorrectly specified structures and often rely on iterative heuristics, where errors may accumulate or drift, resulting in limited controllability under candidate constraints.

\textbf{Feature selection under weak/ambiguous supervision.}
Previous work has included embedded/joint feature selection for multi label or weakly supervised learning \cite{wu2025partial,hao2025embedded}, as well as budget sequential feature selection for policy learning formulation \cite{liu2021automated}. However, pseudo label noise can severely distort feature selection, and most sequential methods assume relatively clean supervision, leading to the problem of error propagation from disambiguation quality to downstream generalization often being ignored.

In contrast, we transform PML disambiguation into candidate constraint horizon-1 POMDP and unify it with budget sequential feature selection to generate controllable hard pseudo labels and end-to-end decision consistency.

\section{Preliminaries}

\subsection{Partial Multi-Label Learning}
In this section, we will formally describe the partial multi-label feature selection problem studied in this paper.Let \(X \in \mathbb{R}^d\) denote the feature space and  \(Y = \{0,1\}^L\) the label space where \(Y = (Y_1,\dots,Y_L) \in \{0,1\}^L\)
is the unobserved ground-truth vector.

In the partial multi-label scenario, the real label space \(Y\) cannot be directly observed. Each input X will be assigned a candidate label set \( C(X) \in L = \{1,\dots,L\}\), at least one of which is a true positive label. For other negative labels that are not in the candidate label set, they can be considered absolutely correct.

\subsection{Partially Observable Markov Decision Processes}
In this chapter, we will briefly introduce Markov Decision Processes (MDPs) and Partially Observable Markov Decision Processes (POMDPs) used in our approach. The MDP is a tuple: 
\begin{equation}
\mathcal{M}=(\mathcal{S}, \mathcal{A}, P, r, \gamma),
\end{equation}
where  \(\mathcal{S}\) is the state space, \(\mathcal{A}\) the action space, \(P\) the transition kernel, \(r\) the reward function, and \(\gamma\) a discount factor. At each time step \(t\), the agent will observe and select the corresponding action \(A_t\) according to the current status \(S_t\) , obtain the corresponding reward \(R_t\) , and the status \(S_t\)  will also be transferred to \(S_{t+1}\).

in the Partially Observable Markov Decision Processes (POMDPs) situation, the true state \(S_t\) cannot be directly observed. Agents can only obtain observation \(O_t\) and can only choose actions based on past observations and actions.A classic policy
\(\pi\) for a POMDP maps observations to distributions over actions, and its expected return is:
\begin{equation}
J(\pi):=\mathbb{E}_{S, O, A \sim \pi}[r(S, A)] .
\end{equation}

\section{Method Proposed}
In the POMDP paradigm, the real environment cannot be completely observed, which is very similar to the PML problem where there is no real label set can be observed but only candidate label set. Based on this common ground, we present our two-stage POMDP framework for partial multi-label feature selection under candidate label sets. In the first stage, we use a horizon-1 POMDP for hard disambiguation and learns a stochastic policy to converge candidate sets; In the second stage, we use the pseudo label set obtained in the first stage as supervision and perform feature selection. The pseudocode is provided in Appendix A.
\subsection{Stage 1: Horizon-1 POMDP for hard label disambiguation}
We first model the hard disambiguation problem as a horizon-1 POMDP defined over the joint distribution of partial multi-label data. This allows us to factor the PML setting into the standard ingredients of a POMDP and to interpret a disambiguation rule as a policy. 

Let \(X \in \mathbb{R}^d\) be the feature vector, \(Y = \{0,1\}^L\) the label space where \(Y = (Y_1,\dots,Y_L) \in \{0,1\}^L\) be the unobserved ground-truth vector, and \( C(X) \in L = \{1,\dots,L\}\) be the candidate label set.The latent state of the environment is
is the unobserved ground-truth vector.  The latent state \(S\) of the environment is:
\begin{equation}
S=(X,Y,C),
\end{equation}
drawn from the underlying joint distribution induced by the PML data. While the agent can only get observation  \(O\):
\begin{equation}
O=(X,C),
\end{equation}
which is exactly the information available in the PML setting. The action \(A\) of the agent is a label vector:
\begin{equation}
A = Z \in \{0,1\}^L,
\end{equation}
\(Z\) is the pseudo label vector after disambiguation. After choosing 
\(Z\), the agent receives a one-step reward:
\begin{equation}
 r(S,Z) \;:=\; -\,\ell_{\mathrm{disc}}(X,C,Z),
\end{equation}
where $\ell_{\mathrm{disc}}$ is a label-free disambiguation surrogate used in our implementation. In our implementation,  \(\ell_{\mathrm{disc}}\) is the sum of a multi-label logistic loss and a structural regularizer.
\begin{equation}
\begin{aligned}
&\ell_{\text {disc }}(X,C)
=\frac{1}{L} \sum_{j=1}^L \ell_{\mathrm{BCE}}\left(U_j(X, C), Z_j\right) \\
&\quad +\lambda_{\text {struct }} \mathcal{R}_{\text {graph }}
\left(\left\{X_i\right\},\left\{U\left(X_i, C_i\right)\right\}\right),
\end{aligned}
\end{equation}
where \(U(X, C) \in \mathbb{R}^L\) is the logits of a discriminative head, \( \ell_{\mathrm{BCE}}\) is the binary cross-entropy between logits and hard pseudo labels, and \(\mathcal{R}_{\text {graph }}\) is a Laplacian regularizer.

As we use the horizon-1 POMDP, each iteration is composed of Independent states \(S=(X,Y,C)\), observations \(O=(X,C)\), actions \(Z\) and rewards \(r(S, Z)\). A stochastic policy \(\pi_\theta\) maps the observations to the distribution over the pseudo hard label. Its expected return is:

\begin{equation}
\begin{aligned}
&J(\pi):=\mathbb{E}_{(S, O, Z) \sim \pi}[r(S, Z)] \\
&=-\mathbb{E}_{(X,C)}\left[\ell_{\text {disc }}\left(X,C,Z\right)\right],
\end{aligned}
\end{equation}

where \(g_\theta\) denotes the hard disambiguation rule induced by \(\pi_\theta\). In 5.1, we proved that maximizing \(J(\theta)\) is equivalent to  minimizing partial multi-label disambiguation risk. Therefore stage 1 can  be viewed as learning an optimal disambiguation policy in the scenario of horizon-1 POMDP.

We implement policy \(\pi_\theta\) through a dual head encoder. One is policy head, and the other is discriminative head. We use the observation \(O=(X,C)\), as the input to the transformer to get the output representation \(\phi(O)\in \mathbb{R}^d\), which is the token of the final layer, and then feed it to both heads. The policy head will implement Bernoulli factorization for all candidate labels: for each candidate label c:
\begin{equation}
p_{\theta, j}(O)=\sigma\left(\theta_j^{\top} \phi(O)\right),
\end{equation}
By sampling the Bernoulli variables with parameter \(p_{\theta, j}(O)\), we obtain the probability \(\pi_\theta(Z \mid O)\) representing the accuracy of the candidate label. Then we can obtain the pseudo label \(\hat{Y}=g_\theta(X, C)\) by a threshold \(Z_j=\mathbb{I}\left[p_{\theta, j}(O) \geq 1 / 2\right]\).

The discriminative head receives \(\phi(O)\) and outputs Logits \(U\), which is used to calculate a loss \(\ell_{\mathrm{disc}}\)  based on pseudo label \(Z\). For every batch, we update encoder and discriminative head parameters by minimizing \(\ell_{\mathrm{disc}}\). 

For the policy header, we take the negative discrimination loss \(\ell_{\mathrm{disc}}\) as the reward signal. For each batch:
\begin{equation}
r^{\mathcal{B}}=-\ell_{\mathrm{disc}}^{\mathcal{B}},
\end{equation}
The policy head parameters are updated using a reinforcement objective:
\begin{equation}
\mathcal{L}_{\mathrm{pol}}(\theta)=-r^B \log \pi_\theta^B,
\end{equation}
which corresponds to a stochastic gradient ascent step on 
\(J(\theta)\). In implementation, we detach the encoder representation \(\phi(O)\) when computing \(\mathcal{L}_{\mathrm{pol}}\), so that the encoder and discriminative head are updated only through \(\ell_{\mathrm{disc}}^{\mathcal{B}}\), while the policy head is updated only through \(\mathcal{L}_{\mathrm{pol}}\).

\subsection{Stage-2: RL-Based Sequential Feature Selection}
Given the hard pseudo labels generated in Stage 1, Stage 2 learns to select a feature subset through RL agent. Feature selection is modeled again as a dual head POMDP process over feature space: starting from an empty mask, the agent iteratively selects features to be revealed, and finally evaluates the resulting subset through the supervised loss calculated under the pseudo labels.

For an input \((X,Y^{ps})\) where \(Y^{ps}\) is set of pseudo labels. With a fixed feature budget \(k_{\mathrm{fs}}\), we define a finite-horizon POMDP of length:
\begin{equation}
T=\min \left(k_{\mathrm{fs}}, d\right) .
\end{equation}
At decision step \(t\), a binary mask \(m_t \in \{0,1\}^d\) is used to indicate the selected features. The underlying state at step \(t\) can be viewed as a tuple \(S_t=(X,Y^{ps},m_t)\) while the agent observes a partially revealed feature vector:
\begin{equation}
x_t=m_t \odot x,
\end{equation}
where \(\odot\) denotes element-wise multiplication and unselected coordinates are set to zero. Encoder implemented in Stage 1 is reused to map \(x_t\) to a representation \(h_t\). 

In each step, the agent's action is selecting a new feature index. A feature selection policy head \(\pi_\psi\) maps \(h_t\) to logits on the feature index to select the new feature by sampling \(a_t\) from \(\pi_\psi\left(\cdot \mid h_t\right)\). While the selected features are masked by adding a large negative constant to their logits. 

Finally, the agent would select a subset of features encoded by \(m_T\). Then the observation \(m_T\) would be mapped to representation \(h_T\) to be fed into the discriminative head. A supervised loss is defined against the pseudo labels by the output of the discriminative head:
\begin{equation}
\ell_{\mathrm{sup}}\left(x, y^{\mathrm{ps}}, m_T\right)=\frac{1}{L} \sum_{j=1}^L \ell_{\mathrm{BCE}}\left(\hat{s}_j, y_j^{\mathrm{ps}}\right),
\end{equation}
and the terminal reward of the trajectory is set as the negative loss. Therefore, the subset that allows accurate prediction of pseudo labels will get higher rewards, while the selection that is not informative or redundant will be punished.

The trajectory of a single sample is a sequence \((h_0,a_0,...,h_{T-1},a_{T-1},h_{T})\). A feature-selection policy \(\pi_\psi\) is trained to maximize the expected terminal reward. We use a standard policy-gradient approach: for each trajectory, we form the log-probability of the action sequence:
\begin{equation}
\log \pi_\psi(\tau)=\sum_{t=0}^{T-1}\log \pi_\psi(a_t\mid h_t).
,
\end{equation}
and update \(\psi\) in the ascent direction of the objective:
\begin{equation}
J(\psi)=\mathbb{E}\left[r\left(x, y^{\mathrm{ps}}, m_T\right)\right]
\end{equation}
using the reinforcement estimator. In practice, we use:
\begin{equation}
\mathcal{L}_{\mathrm{pol}}^{\mathrm{fs}}(\psi)=-\mathbb{E}\left[(r-b) \log \pi_\psi(\tau)\right],
\end{equation}
Where \(b\) is the scalar baseline (the average of recent rewards) used to reduce variance. Encoder and  head are trained by a joint objective function:
\begin{equation}
\mathcal{L}_{\text {stage2 }}=\mathbb{E}\left[\ell_{\text {sup }}\left(x, y^{\mathrm{ps}}, m_T\right)\right]+\lambda_{\mathrm{fs}} \mathcal{L}_{\mathrm{pol}}^{\mathrm{fs}}(\psi),
\end{equation}

where \(\lambda\) controls the weight of RL items. In the experiment, we use the encoder and discriminative head extracted from the convergent first stage model,  and then fine tune them together with \(\pi_\psi\), so as to stabilize the training empirically to obtain better performance.

In the experiment, we interpret the learned strategy as a feature selector. We rank features by averaging the scores over the training data, which will produce a global ranking; For any required budget \(K\), we select the first \(K\) features and train downstream classifiers on these subsets, such as SVM or MLKNN. In this way, Stage 2 provides a label aware feature ranking guided by the pseudo labels constructed in Stage 1.

\section{Theoretical Analysis}

In this section, we provide a concise theoretical analysis for our two-stage framework. First, we prove that the horizon-1 POMDP introduced in Section 4.1 is equivalent to the partial multi-label disambiguation problem, so that optimizing the one-step reward for is consistent with minimizing the risk of PML disambiguation (Theorem 5.2). Next, we prove that under mild regularity conditions, the RL agent used in the first stage converges to the first-order stationary point of the surrogate objective, resulting in a locally optimal hard disambiguation strategy (Theorem 5.3). Finally, we derive the excess-risk bound for prediction using the selected features, where the error is decomposed into a term controlled by the pseudo label quality in the first stage and a standard generalization term determined by the sparsity level and sample size in the second stage (Theorem 5.5).
\subsection{Equivalence between horizon-1 POMDP and PML disambiguation}
We now formalize the connection between the horizon-1 POMDP introduced in Section 4.1 and the partial multi-label disambiguation problem. 
\begin{definition}
Given the surrogate loss \(\ell_{\mathrm{disc}}\), for any disambiguation rule \(g\) that maps an observation \((X,C)\) to hard pseudo label vector \(Z\), we define its risk of partial multi-label disambiguation as:
\begin{equation}
\mathcal{R}_{\mathrm{PML}}(g)=\mathbb{E}_{(X, C)}\left[\ell_{\text {disc }}\left(X,C,Z\right)\right].
\end{equation}
\end{definition}

Let \(\Pi\) denote the class of stochastic policies \(\pi\) that maps each observation \(O\) to a probability distribution over \(A(X,C)\). For any \(\pi\in\Pi\), its expected return in this horizon-1 POMDP is:
\begin{equation}
J(\pi):=\mathbb{E}_{(S, O, Z) \sim \pi}[r(S, Z)].
\end{equation}

\begin{theorem}[Equivalence between PML risk and horizon-1 POMDP return]
Assume that \(\ell_{\mathrm{disc}}(X, C)\) is measurable and integrable with respect to the joint distribution of \((X,C)\) by any \(\pi\in\Pi\):

(a) For every stochastic policy  \(\pi\in\Pi\) , the mapping below defines a disambiguation rule
\begin{equation}
g_\pi(X, C):=Z, \quad Z \sim \pi(\cdot \mid X, C)
\end{equation}
of which disambiguation risk is consistent with the negative POMDP return:
\begin{equation}
J(\pi)=-\mathcal{R}_{\mathrm{PML}}\left(g_\pi\right) .
\end{equation}

(b) Correspondingly, for each disambiguation rule \(g\), there is a policy \(\pi\in\Pi\) such that  \(g=g_{\pi_g}\) and \(J(\pi)=-\mathcal{R}_{\mathrm{PML}}\left(g\right)\)

(c) Let \(\mathcal{R}_{\mathrm{PML}}^\star:=\inf_{g\in\mathcal{G}} \mathcal{R}_{\mathrm{PML}}(g)\) and \(J^\star:=\sup_{\pi\in\Pi} J(\pi)\). Then:
\begin{equation}
J^\star = -\mathcal{R}_{\mathrm{PML}}^\star,
\end{equation}
and any Bayes-optimal disambiguation rule \(g^\star \in \arg\min_{g\in\mathcal{G}} \mathcal{R}_{\mathrm{PML}}(g)\) can be realized (up to a set of measure zero) by some optimal policy \(\pi^\star\in\arg\max_{\pi\in\Pi} J(\pi)\), and vice versa.

\end{theorem}
To sum up, solving PML disambiguation is equivalent to solving horizon-1 POMDP: minimizing disambiguation risk \(\mathcal{R}_{\mathrm{PML}}(g) \) is the same as \(j (\pi)\) which maximizes the expected one-step reward on any policy, and Bayes optimal hard disambiguator is consistent with the optimal strategy of horizon-1 POMDP. A fully detailed proof is given in Appendix B.

\subsection{Convergence of Stage-1 policy gradient}

Using the Theorem 5.2, the Stage 1 policy \(\pi_\theta\) derives a hard disambiguation rule \(g_\theta\), and the risk of its partial multi-label disambiguation can be written as: 
\begin{equation}
\mathcal{R}_{\mathrm{PML}}(\theta)
:= \mathcal{R}_{\mathrm{PML}}(g_\theta)
= \mathbb{E}_{(X, C)}\left[\ell_{\text {disc }}\left(X, C, Z\right)\right],
\end{equation}
where the expected value takes the PML distribution \((X, Y, C)\) and the randomness of the policy \(\pi _\theta (Z\mid X, C)\). In Stage 1, we parameterize \(\pi_\theta\) through the policy head acting on the fixed representation of \((X, C)\), and update \(\theta\) through reinforcement policy.

\begin{theorem}[Convergence to a first-order stationary point]

Let \({\theta_t}\) be the sequence of policy parameters produced by the Stage-1 updates:
\begin{equation}
\theta_{t+1}
= \theta_t - \alpha_t g_t,
\end{equation}
where \(\alpha_t>0\) is the step size at iteration \(t\) and
\begin{equation}
g_t
= \ell_{\mathrm{disc}}(X_t,C_t,Z_t)\nabla_\theta \log \pi_\theta(Z_t\mid X_t,C_t)
\Big|_{\theta=\theta_t},
\end{equation}
is the reinforcement gradient estimator constructed from an i.i.d.\ sample \((X_t,C_t)\) and an action \(Z_t \sim \pi_{\theta_t}(\cdot\mid X_t,C_t).\). Assume that:

(a) the surrogate risk $\mathcal{R}_{\mathrm{PML}}(\theta)$ is bounded below and continuously differentiable on a closed convex parameter set $\Theta\subset\mathbb{R}^p$, and its gradient $\nabla_\theta \mathcal{R}_{\mathrm{PML}}(\theta)$ is Lipschitz-continuous on $\Theta$.

(b) the stochastic gradients are unbiased and have bounded second moments, for some finite constant (G\textgreater0) and all \(t\),
\begin{equation}
\mathbb{E}[g_t \mid \theta_t] = \nabla_\theta \mathcal{R}_{\mathrm{PML}}(\theta_t),
\quad
\mathbb{E}\big[||g_t||^2 \mid \theta_t\big] \le G^2
\end{equation}

(c) the step sizes satisfy the Robbins--Monro conditions
\begin{equation}
\sum_{t=0}^\infty \alpha_t = \infty,
\quad
\sum_{t=0}^\infty \alpha_t^2 < \infty.
\end{equation}
Then \(\mathcal{R}_{\mathrm{PML}}(\theta_t)\) converges almost surely, and
\begin{equation}
\lim_{t\to\infty} \big|\big|\nabla_\theta \mathcal{R}_{\mathrm{PML}}(\theta_t)\big|\big| = 0
\end{equation}
is almost surely.

\end{theorem}
In particular, each limit point \(\{\theta_t\}\) of the sequence is the first-order stationary point of \(\mathcal{R}_{\mathrm{pml}}\), and the corresponding strategy \(\pi_{\theta_ t}\) is a locally optimal hard disambiguation strategy in the sense of Theorem 5.2.  The fully detailed proof is provided in Appendix C.

\subsection{Generalization risk under pseudo-label supervision}

We now analyze the impact of using pseudo labels generated by Stage 1 in Stage 2 on downstream classifiers. Let \(\ell_{\mathrm{cls}}\) denote the loss of multi label classification for the training Stage 2 classifier. For the classifier $f$(e.g., the classifier at the top of the selected feature), we define its overall risk under the real label as:
\begin{equation}
\mathcal{R}_{\mathrm{cls}}(f)
:= \mathbb{E}\big[\ell_{\mathrm{cls}}(f(X),Y)\big].
\end{equation}
Then if the classifier uses the hard pseudo labels, the risk is:
\begin{equation}
\mathcal{R}_{\mathrm{cls}}^{(g)}(f)
:= \mathbb{E}\big[\ell_{\mathrm{cls}}(f(X),Z)\big],
\end{equation}
In order to clarify the impact of disambiguation errors, we introduce the average pseudo label noise level \(\varepsilon_{\mathrm{pseudo}}(g)\) based on the normalized Hamming distance between real labels and pseudo labels:
\begin{equation}
\varepsilon_{\mathrm{pseudo}}(g)
:= \mathbb{E}\big[d_{\mathrm{H}}(Y, Z)\big],
\end{equation}
where \(d_{\mathrm{H}}\) is the Hamming distance on \(\{0,1\}^L\) scaled to lie in \([0,1]\). Intuitively, \(\varepsilon_{\mathrm{pseudo}}(g)\) captures the average error disambiguation frequency of disambiguation rules.

\begin{assumption}

We assume that \(\ell_{\mathrm{cls}}\) is Lipschitz in its label argument with respect to this distance. Then the loss \(\ell_{\mathrm{cls}}(\cdot,\cdot)\) is bounded by \(M\) for some constant $M>0$ and there exists a constant \(L_{\mathrm{cls}}>0\) such that, for all predictions \(\hat{y}\in[0,1]^L\) and all label vectors \(y,y'\in \{0,1\}^L\),
\begin{equation}
\big|\ell_{\mathrm{cls}}(\hat{y},y) - \ell_{\mathrm{cls}}(\hat{y},y')\big|
\le L_{\mathrm{cls}}\,d_{\mathrm{H}}(y,y').
\end{equation}
Let \(\mathcal{F}\) be the hypothesis class used for the downstream classifier after Stage 2. Given \(n\) i.i.d. training examples \(\{(X_i,Y_i,C_i)\}_{i=1}^n\) and pseudo labels \(Z_i=g(X_i,C_i)\), we consider an empirical pseudo-label risk minimizer
\begin{equation}
\hat{f} \in \arg\min_{f\in\mathcal{F}} \widehat{\mathcal{R}}{\mathrm{cls}}^{(g)}(f),
\end{equation}
and compare its performance to the best-in-class predictor
\begin{equation}
f^\star \in \arg\min_{f\in\mathcal{F}} \mathcal{R}_{\mathrm{cls}}(f).
\end{equation}
\end{assumption}

We now derive a generalization bound for the Stage 2 classifier trained on pseudo labels.
\begin{theorem}[Excess risk bound for pseudo-label training]
Let Assumption~5.4 and Assumption~D.1 hold (see Appendix~D). Let \(\mathcal{F}\) be any hypothesis class of predictors \(f:\mathcal{X}\to[0,1]^L\). For any disambiguation rule \(g\) and any \(\delta\in(0,1)\), with probability at least \(1-\delta\) over the draw of the sample \(\{(X_i,Y_i,C_i)\}_{i=1}^n\), the empirical pseudo label minimizer \(\hat{f}\) satisfies:
\begin{equation}
\mathcal{R}_{\mathrm{cls}}(\hat{f}) - \mathcal{R}_{\mathrm{cls}}(f^\star)
\le2L_{\mathrm{cls}}\varepsilon_{\mathrm{pseudo}}(g)
+2\mathrm{Gen}_n(\mathcal{F},\delta),
\end{equation}
where \(\mathrm{Gen}_n(\mathcal{F},\delta)\) is a standard generalization complexity term that depends only on \(\mathcal{F}\), \(n\), and \(\delta\), but not on \(g\). 
\end{theorem}
In words, Theorem 5.5 shows that the excess risk of \(\hat{f}\) 
based on the pseudo labels consists of two parts: a pseudo-label error term, which depends only on the disambiguation rule \(g\) through \(\varepsilon_{\mathrm{pseudo}}(g)\) and a generalization term \(\mathrm{Gen}_n(\mathcal{F},\delta)\) which depends only on the complexity of \(\mathcal{F}\) and the sample size \(n\). A detailed proof is given in Appendix D.

\section{Experiments}
\subsection{Experimental Setup}

\begin{table*}[htbp]
    \centering
  
    \label{tab:full_results}
    
    \scriptsize
    \setlength{\tabcolsep}{1.8pt}
    \renewcommand{\arraystretch}{0.82}
    \caption{Experimental results (mean ± std.) on 9 datasets showing Ranking Loss (RL $\downarrow$) and Micro-F1 (Micro-F1 $\uparrow$).}
    \begin{tabular}{@{}l*{10}{c}@{}}
        \toprule
        \textbf{Datasets} & \textbf{\textsc{POMDP-FS}} & \textbf{\textsc{PML-FSLA}} & \textbf{\textsc{PML-FSMIR}} & \textbf{\textsc{PML-FSSO}} & \textbf{\textsc{fPML}} & \textbf{\textsc{PML-LD}} & \textbf{\textsc{PAMB}} & \textbf{\textsc{PML-VLS}} & \textbf{\textsc{PML-MAP}} \\ 
        \midrule
        \multicolumn{10}{c}{\textbf{RL} ($\downarrow$)} \\ 
        \midrule
          Bibtex      & 0.16$\pm$0.05 & 0.29$\pm$0.00 & 0.17$\pm$0.08 & \textbf{0.13$\pm$0.15} & 0.44$\pm$0.09 & 0.21$\pm$0.06 & 0.37$\pm$0.27 & 0.46$\pm$0.22 & 0.60$\pm$0.16 \\
        Birds       & \textbf{0.09$\pm$0.01} & 0.26$\pm$0.04 & 0.32$\pm$0.09 & 0.33$\pm$0.08 & 1.00$\pm$0.00 & 0.39$\pm$0.01 & 0.50$\pm$0.16 & 0.59$\pm$0.12 & 0.69$\pm$0.05 \\
        CHD\_49     & \textbf{0.22$\pm$0.00} & 0.33$\pm$0.05 & 0.31$\pm$0.19 & 0.31$\pm$0.19 & 0.72$\pm$0.08 & 0.44$\pm$0.10 & 0.47$\pm$0.11 & 0.34$\pm$0.14 & 0.61$\pm$0.06 \\
        Chess       & \textbf{0.12$\pm$0.00} & 0.15$\pm$0.05 & 0.15$\pm$0.07 & 0.22$\pm$0.18 & 0.87$\pm$0.04 & 0.3$\pm$0.08 & 0.42$\pm$0.24 & 0.45$\pm$0.14 & 0.40$\pm$0.16 \\
        HumanPseAAC & \textbf{0.17$\pm$0.01} & 0.25$\pm$0.00 & 0.25$\pm$0.02 & 0.49$\pm$0.25 & 0.85$\pm$0.07 & 0.51$\pm$0.03 & 0.30$\pm$0.05 & 0.46$\pm$0.28 & 0.61$\pm$0.17 \\
        Mediamill   & \textbf{0.07$\pm$0.01} & 0.14$\pm$0.00 & 0.15$\pm$0.00 & 0.13$\pm$0.15 & 0.44$\pm$0.09 & 0.68$\pm$0.09 & 0.24$\pm$0.12 & 0.46$\pm$0.22 & 0.60$\pm$0.16 \\
        PlantPseAAC & \textbf{0.22$\pm$0.00} & 0.32$\pm$0.00 & 0.30$\pm$0.03 & 0.29$\pm$0.00 & 0.81$\pm$0.08 & 0.71$\pm$0.08 & 0.53$\pm$0.03 & 0.57$\pm$0.23 & 0.73$\pm$0.15 \\
        Slashdot    & \textbf{0.03$\pm$0.00} & 0.06$\pm$0.05 & 0.05$\pm$0.04 & 0.05$\pm$0.07 & 1.00$\pm$0.00 & 0.03$\pm$0.02 & 0.54$\pm$0.21 & 0.44$\pm$0.27 & 0.43$\pm$0.27 \\
        Yeast       & \textbf{0.19$\pm$0.01} & 0.22$\pm$0.00 & 0.50$\pm$0.03 & 0.35$\pm$0.14 & 0.34$\pm$0.12 & 0.31$\pm$0.02 & 0.31$\pm$0.04 & 0.34$\pm$0.14 & 0.62$\pm$0.01 \\
        \midrule
        \multicolumn{10}{c}{\textbf{Micro-F1} ($\uparrow$)} \\ 
        \midrule
        Bibtex      & 0.48$\pm$0.12 & 0.14$\pm$0.00 & 0.27$\pm$0.06 & 0.17$\pm$0.05 & \textbf{0.57$\pm$0.05} & 0.43$\pm$0.17 & 0.39$\pm$0.16 & 0.35$\pm$0.19 & 0.18$\pm$0.05 \\
        Birds       & \textbf{0.45$\pm$0.05} & 0.30$\pm$0.01 & 0.38$\pm$0.03 & 0.33$\pm$0.09 & 0.00$\pm$0.00 & 0.29$\pm$0.17 & 0.06$\pm$0.05 & 0.20$\pm$0.10 & 0.29$\pm$0.09 \\
        CHD\_49     & \textbf{0.68$\pm$0.00} & 0.64$\pm$0.06 & 0.60$\pm$0.19 & 0.19$\pm$0.13 & 0.12$\pm$0.07 & 0.08$\pm$0.06 & 0.16$\pm$0.10 & 0.16$\pm$0.12 & 0.24$\pm$0.12 \\
        Chess       & \textbf{0.69$\pm$0.00} & 0.58$\pm$0.06 & 0.13$\pm$0.02 & 0.04$\pm$0.00 & 0.05$\pm$0.00 & 0.42$\pm$0.13 & 0.47$\pm$0.19 & 0.50$\pm$0.13 & 0.54$\pm$0.15 \\
        HumanPseAAC & \textbf{0.23$\pm$0.04} & 0.19$\pm$0.00 & 0.18$\pm$0.07 & 0.01$\pm$0.00 & 0.00$\pm$0.00 & 0.01$\pm$0.01 & 0.05$\pm$0.04 & 0.15$\pm$0.09 & 0.19$\pm$0.11 \\
        Mediamill   & \textbf{0.63$\pm$0.04} & 0.29$\pm$0.10 & 0.00$\pm$0.00 & 0.00$\pm$0.00 & 0.14$\pm$0.03 & 0.00$\pm$0.00 & 0.01$\pm$0.00 & 0.00$\pm$0.00 & 0.00$\pm$0.00 \\
        PlantPseAAC & 0.10$\pm$0.06 & 0.18$\pm$0.01 & 0.00$\pm$0.00 & 0.00$\pm$0.00 & 0.00$\pm$0.00 & 0.09$\pm$0.05 & 0.03$\pm$0.02 & \textbf{0.23$\pm$0.08} & 0.06$\pm$0.08 \\
        Slashdot    & \textbf{0.89$\pm$0.00} & 0.18$\pm$0.00 & 0.81$\pm$0.09 & 0.00$\pm$0.00 & 0.00$\pm$0.00 & 0.37$\pm$0.25 & 0.3$\pm$0.11 & 0.40$\pm$0.23 & 0.44$\pm$0.20 \\
        Yeast       & \textbf{0.62$\pm$0.02} & 0.50$\pm$0.02 & 0.48$\pm$0.16 & 0.03$\pm$0.03 & 0.00$\pm$0.00 & 0.00$\pm$0.00 & 0.02$\pm$0.01 & 0.00$\pm$0.00 & 0.00$\pm$0.00 \\
        \bottomrule
    \end{tabular}%
\end{table*}

We evaluated our method on nine widely used partial multi-label datasets covering different domains, including an audio dataset (Birds), three biological datasets (i.e. HumanPseAAC, Yeast, and PlantPseAAC), a medical dataset (i.e. CHD\_49), two text datasets (i.e. Slashdot and Bibtex), and two image datasets (i.e. Mediamill and Chess). Details of the datasets are provided in Appendix E. We randomly injected 20\% annotation noise to construct a partial multi-label candidate set. All results are based on 5-fold cross-validation. 

We compare our approach with three PML feature selection methods (PML-FSLA \cite{pan2025reconsidering}, PML-FSMIR \cite{ijcai2025p576} and PML-FSSO \cite{hao2023partial}) and five PML learning methods (fPML \cite{yu2018feature}, PML-LD \cite{xu2020partial}, PAMB \cite{liu2023towards}, PML-VLS \cite{zhang2020partial} and PML-MAP \cite{zhang2020partial}). For the PML learning method, we derive feature ranking from its learned weight matrix, select the top 20\% of features (or the top 20 features when the original feature dimension is less than 100), and train a unified downstream SVM classifier for fair comparison. Implementation details are provided in the Appendix G. Performance metrics include Ranking Loss, Hamming Loss, Coverage Error, and Micro-F1 (lower RL/HL/CE values are better, and higher Micro-F1 values are better). 
\subsection{Experiment Results}

Table 1 summarizes the main results for Ranking Loss (RL) and Micro-F1 under 5-fold cross-validation. We use POMDP-FS as a shorthand for our approach. On nine test datasets, POMDP-FS achieves the lowest (or tied for lowest) RL values on eight datasets; the only exception is the Bibtex dataset, where PML-FSSO has an even lower RL value. This consistent advantage in RL indicates that our proposed disambiguation feature selection method better preserves the relative order of relevant labels even when the candidate label set is perturbed by noise. The improvement of POMDP-FS is particularly significant on datasets with stronger ambiguity and richer label relevance (e.g., Birds, Mediamill, and Slashdot), with RL values far lower than other methods. In contrast, some PML learning methods show a significant decrease in RL values after converting the learned weight matrix into feature ranking, suggesting that rankings generated through soft disambiguation may lead to wrong direction.

For partially multi-label data with sparse character, the F1 score is a more valuable metric because the ability to distinguish positive labels is more important. In terms of the Micro-F1 metric, POMDP-FS achieves best performance on seven out of nine datasets, indicating that ranking-level advantage often translates to better discrete predictions. This effect is particularly evident on datasets like Slashdot and Mediamill where POMDP-FS achieves high Micro-F1 scores, while some methods show significant performance degradation under noisy candidate labels. The two datasets where POMDP-FS is not best in Micro-F1 (Bibtex and PlantPseAAC) are also instructive. On the Bibtex dataset, fPML has a higher Micro-F1 score, suggesting that our method may require more favorable threshold settings or calibration for hard decisions, as it remains more reliable in ranking. On the PlantPseAAC dataset, PML-VLS has a higher Micro-F1 score, while POMDP-FS still has the lowest RL score, suggesting that performance can be further improved by refining the decision rule based on the robust ranking-guided feature subset. Further results for Coverage Error (CE) and Hamming loss (HL) are provided in the Appendix F, which also confirm the same overall trend.
\subsection{Budget Curves}
\begin{figure}[htbp]
	\centering
    \begin{subfigure}{\linewidth}
		\centering
		\includegraphics[width=\linewidth]{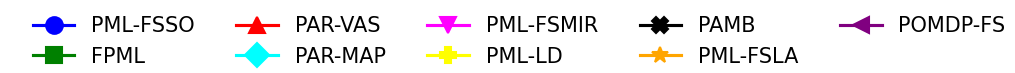}

		\label{ch}
	\end{subfigure}
	\begin{subfigure}{0.49\linewidth}
		\centering
		\includegraphics[width=\linewidth]{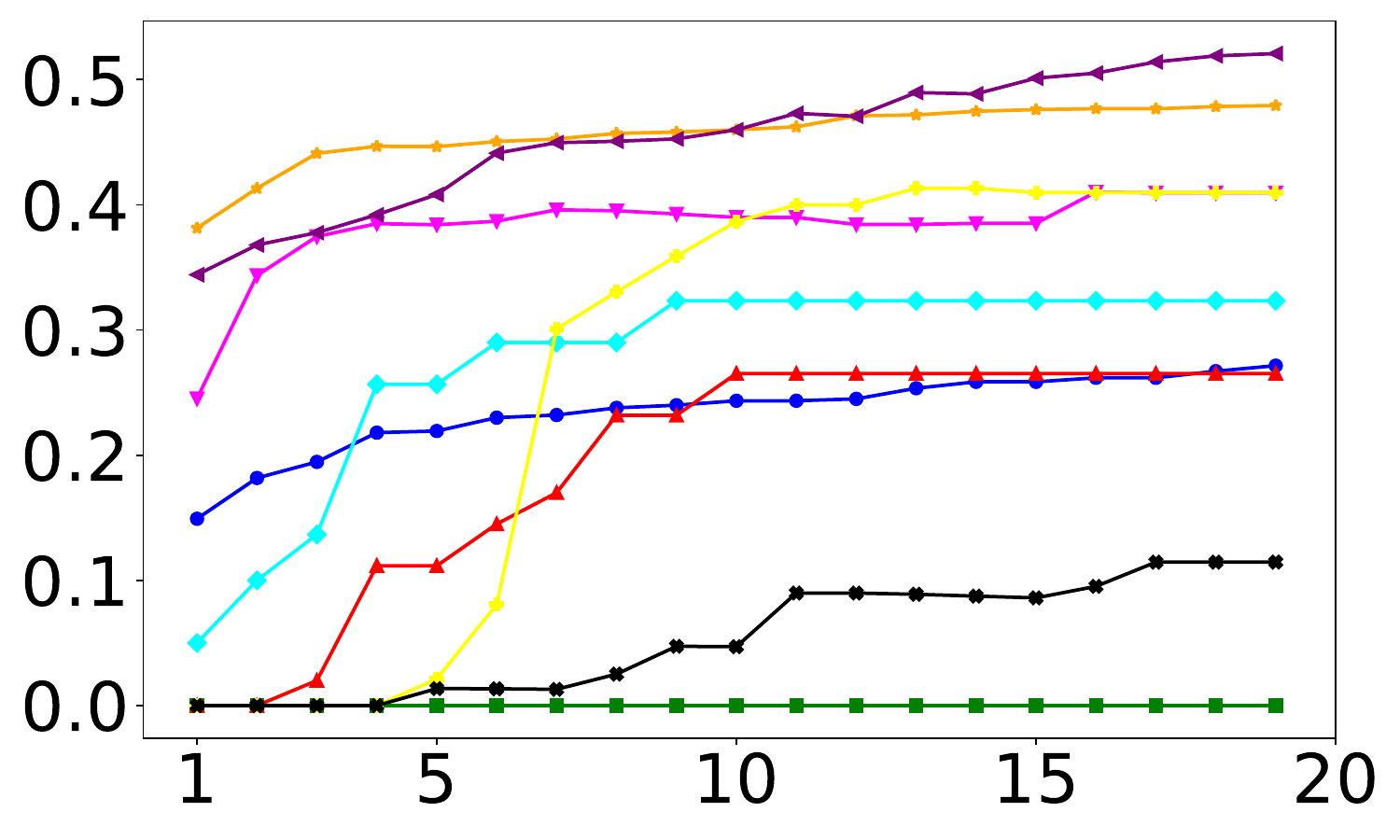}
		\caption{Micro-F1}
		\label{chutian1}
	\end{subfigure}
	\begin{subfigure}{0.49\linewidth}
		\centering
		\includegraphics[width=\linewidth]{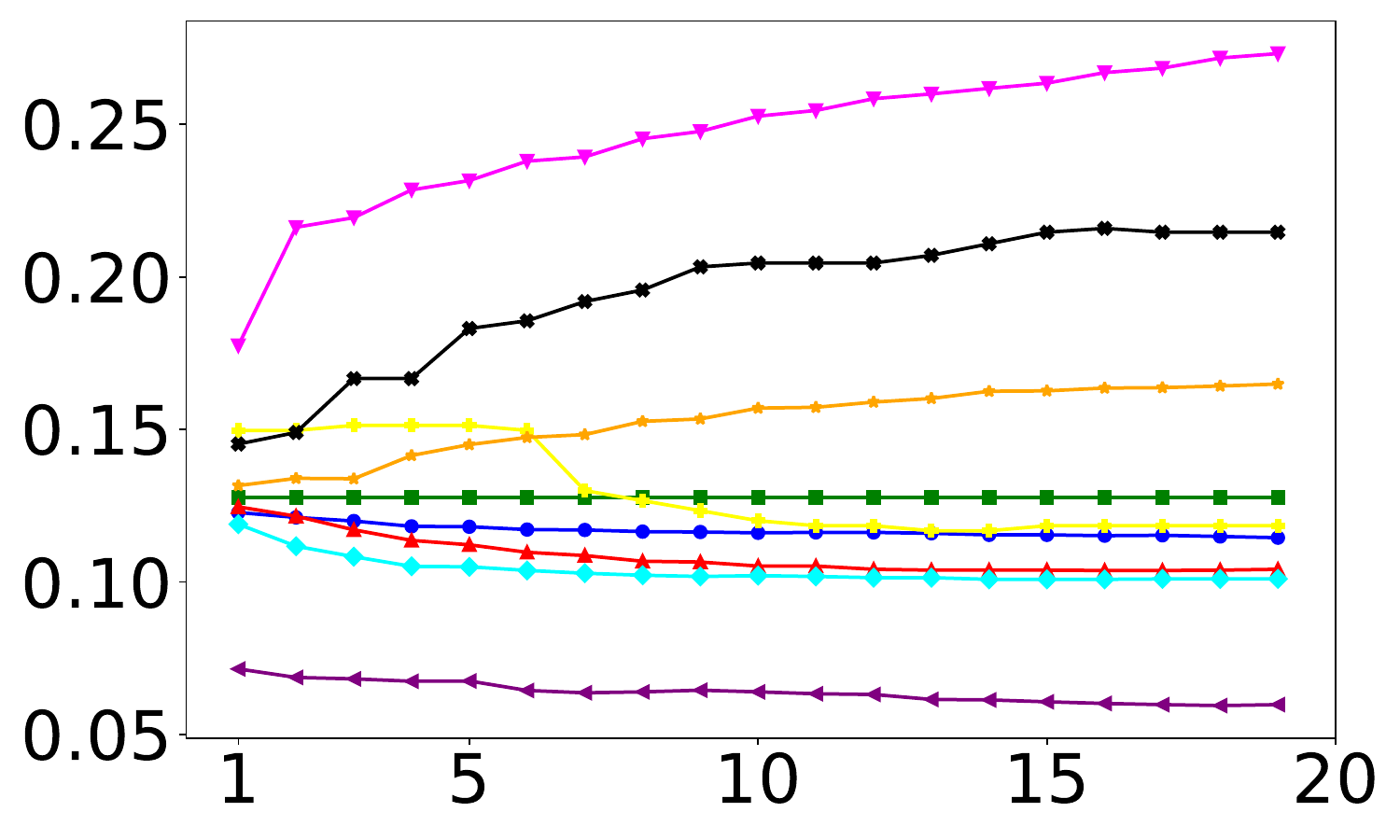}
		\caption{Hamming Loss}
		\label{chutian2}
	\end{subfigure}
	
	\begin{subfigure}{0.49\linewidth}
		\centering
		\includegraphics[width=\linewidth]{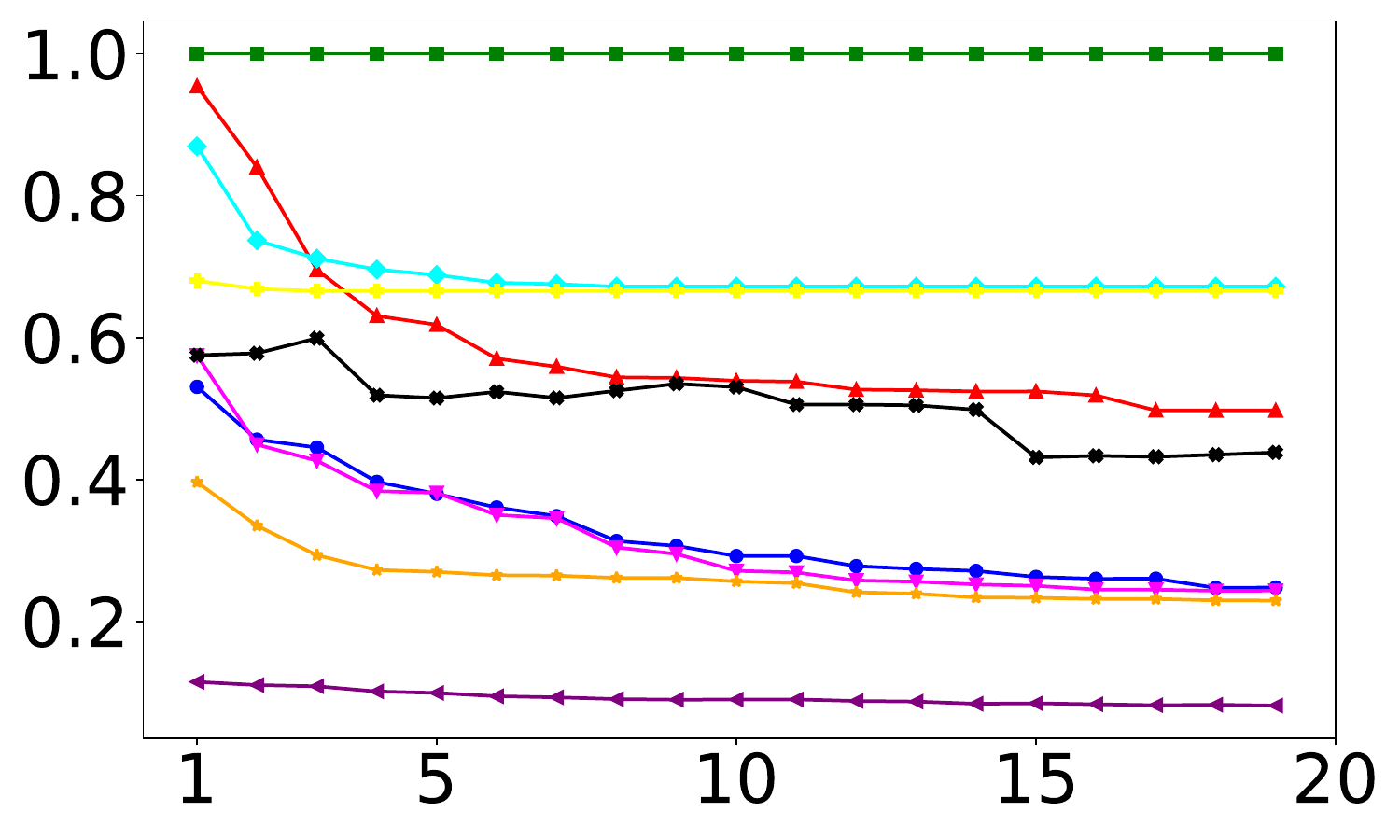}
		\caption{Ranking Loss}
		\label{chutian3}
	\end{subfigure}
	\begin{subfigure}{0.49\linewidth}
		\centering
		\includegraphics[width=\linewidth]{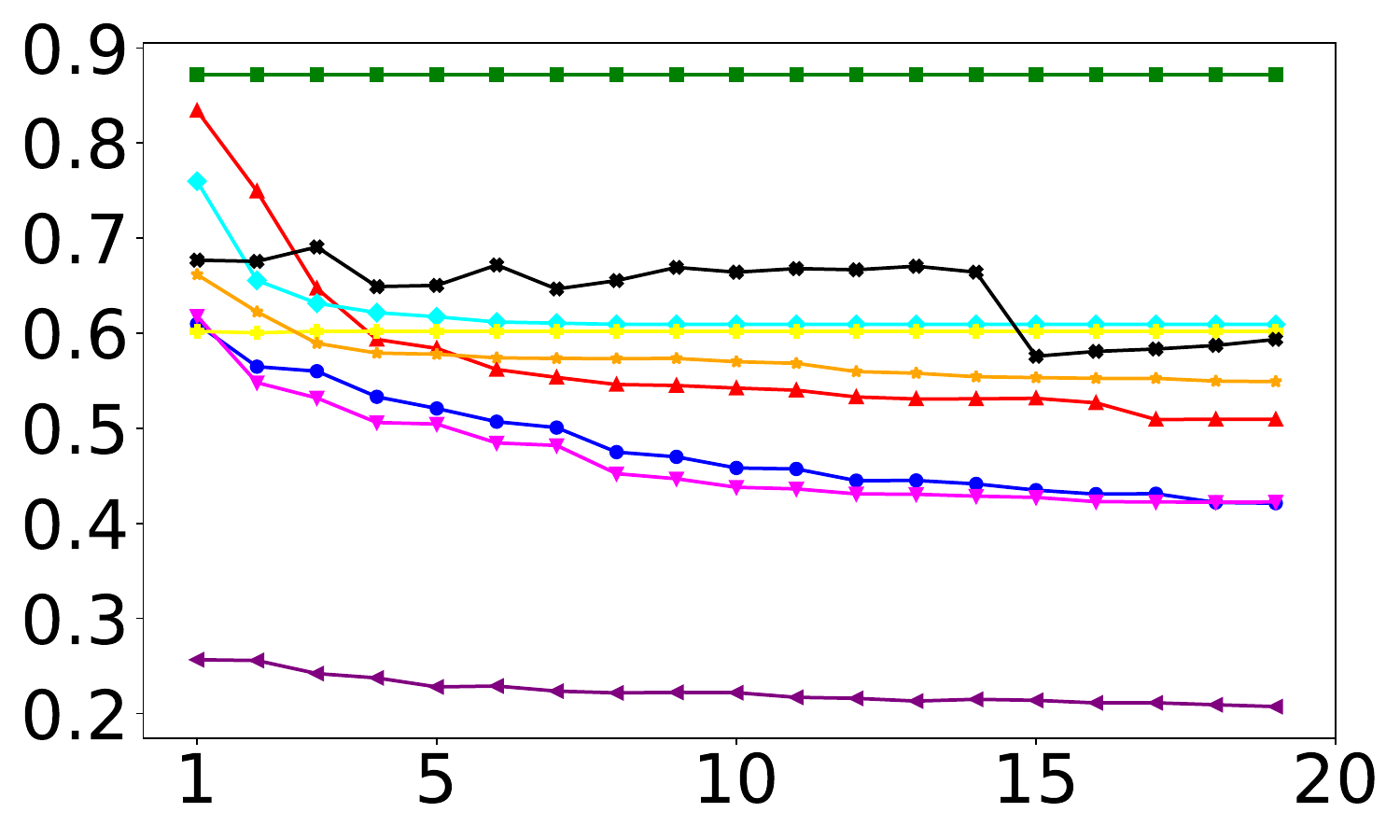}
		\caption{Coverage Error}
		\label{chutian4}
	\end{subfigure}
 \caption{ Nine methods on Birds in terms of Micro-F1, Hamming Loss, Ranking Loss and Coverage Error.}
	\label{f3}
\end{figure}
We further evaluated the impact of feature budget by scanning the top 20\% of budget levels and plotting Micro-F1 ($\uparrow$) and HL/RL/CE ($\downarrow$) curves. For all methods, increasing the budget generally improves performance (Micro-F1 increases, while HL/RL/CE decreases), but the gains diminish with high budgets. Notably, POMDP-FS exhibits the most stable and consistent performance improvement across the entire budget range, especially at low budgets, where selecting informative features is crucial. It maintains significantly lower RL values throughout the scan, indicating that it better preserves label order under feature constraints. Its Micro-F1 value steadily increases with budget, eventually reaching competitive or even optimal levels at high budgets. The budget curves on the other two datasets shown in the Appendix H also show highly similar trends.
\subsection{Statistical Significance}
To assess whether the performance differences between different datasets were statistically significant, we followed a standard nonparametric multi-dataset comparison protocol. For each metric (RL, Micro-F1, HL and CE), we first calculated the ranking of all compared methods on each dataset using average performance (averaging the rankings for those with the same rank), and then reported the average rankings across the nine datasets in Table 2. We further performed Friedman tests on each metric, where \(k = 9\) methods and \(N = 9\) datasets; the corresponding statistics and p-values are summarized in Table 3. In all four cases, the null hypothesis that all methods performed identically was rejected at the conventional significance level, indicating that the observed improvements were unlikely to be due to random differences between the datasets.

\begin{table}[htbp]
    \centering
    \scriptsize
    \setlength{\tabcolsep}{4.2pt}
    \renewcommand{\arraystretch}{0.95}
    \caption{Average ranks over nine datasets (lower is better). Each dataset is ranked using average performance for each metric; in cases of identical rankings, the average ranking is used.}
    \label{tab:avg_rank}
    \begin{tabular}{lccccc}
        \toprule
        \textbf{Method} & \textbf{RL} $\downarrow$ & \textbf{Micro-F1} $\uparrow$ & \textbf{HL} $\downarrow$ & \textbf{CE} $\downarrow$ & \textbf{Avg.} \\
        \midrule
        \textsc{POMDP-FS}   & \textbf{1.167} & \textbf{1.333} & \textbf{1.667} & 3.500 & \textbf{1.917} \\
        \textsc{PML-FSLA}   & 3.333 & 3.556 & 8.111 & 3.278 & 4.569 \\
        \textsc{PML-FSMIR}  & 3.556 & 4.667 & 5.000 & \textbf{1.833} & 3.764 \\
        \textsc{PML-FSSO}   & 3.556 & 6.722 & 4.167 & 3.611 & 4.514 \\
        \textsc{fPML}       & 8.000 & 6.889 & 4.333 & 7.222 & 6.611 \\
        \textsc{PML-LD}     & 5.333 & 6.000 & 6.333 & 5.722 & 5.847 \\
        \textsc{PAMB}       & 5.722 & 5.611 & 6.500 & 6.667 & 6.125 \\
        \textsc{PML-VLS}    & 6.556 & 5.222 & 4.333 & 6.056 & 5.542 \\
        \textsc{PML-MAP}    & 7.778 & 5.000 & 4.556 & 7.111 & 6.111 \\
        \bottomrule
    \end{tabular}
\end{table}

\begin{table}[htbp]
    \centering
    \scriptsize
    \setlength{\tabcolsep}{6.0pt}
    \renewcommand{\arraystretch}{0.95}
    \caption{Friedman test results across nine datasets ($N=9$) and nine methods ($k=9$). We also report the Iman--Davenport corrected $F$ statistic with $(8,\,64)$ degrees of freedom.}
    \label{tab:friedman_test}
    \begin{tabular}{lcccc}
        \toprule
        \textbf{Metric} & $\boldsymbol{\chi^2_F}$ & $\boldsymbol{p}$ & $\boldsymbol{F_F}$ & $\boldsymbol{p}$ \\
        \midrule
        RL        & 49.70 & $4.67\times 10^{-8}$ & 17.83 & $1.16\times 10^{-13}$ \\
        Micro-F1  & 28.32 & $4.17\times 10^{-4}$ & 5.19  & $5.21\times 10^{-5}$ \\
        HL        & 32.19 & $8.63\times 10^{-5}$ & 6.47  & $3.83\times 10^{-6}$ \\
        CE        & 37.18 & $1.07\times 10^{-5}$ & 8.54  & $7.89\times 10^{-8}$ \\
        \bottomrule
    \end{tabular}
\end{table}

\subsection{Ablation Study}
To determine the effectiveness of the proposed Stage 1 hard disambiguation, we  replaced it with a soft disambiguation variant while keeping the rest of the process unchanged. Table 4 reports the results on four representative datasets using Micro-F1 ($\uparrow$), CE ($\downarrow$), and RL ($\downarrow$). Hard disambiguation consistently matches or slightly improves upon soft disambiguation across all three metrics, confirming that making hard decisions can benefit downstream learning when candidate label noise is high. However, the improvements are generally small, meaning this part alone cannot explain most of the overall performance gains. This suggests that the main improvement comes from the overall framework inspired by POMDP and its interaction with subsequent stages. It also highlights an important direction for future work: designing more informative hard disambiguation strategies so that the first-stage commitment can better utilize the label structure.

\begin{table}[htbp]
\centering
\scriptsize
\setlength{\tabcolsep}{3.2pt}
\renewcommand{\arraystretch}{0.92}
\caption{Ablation on stage-1 disambiguation.}
\label{tab:ablation_soft_hard}
\begin{tabular}{@{}lccccc ccccc@{}}
\toprule
 & \multicolumn{5}{c}{\textbf{Bibtex}} & \multicolumn{5}{c}{\textbf{Birds}} \\
\midrule
 & \textbf{Soft} & \textbf{Hard} &  &  &  & \textbf{Soft} & \textbf{Hard} &  &  &  \\
\midrule
Micro-F1 ($\uparrow$) & \textbf{0.48$\pm$0.14} & \textbf{0.48$\pm$0.12} & \multicolumn{3}{c}{} & 0.38$\pm$0.05 & \textbf{0.45$\pm$0.05} & \multicolumn{3}{c}{} \\
CE ($\downarrow$)     & 0.24$\pm$0.05 & \textbf{0.23$\pm$0.05} & \multicolumn{3}{c}{} & 0.23$\pm$0.01 & \textbf{0.22$\pm$0.01} & \multicolumn{3}{c}{} \\
RL ($\downarrow$)     & 0.17$\pm$0.05 & \textbf{0.16$\pm$0.05} & \multicolumn{3}{c}{} & 0.10$\pm$0.01 & \textbf{0.09$\pm$0.01} & \multicolumn{3}{c}{} \\
\midrule
 & \multicolumn{5}{c}{\textbf{HumanPseAAC}} & \multicolumn{5}{c}{\textbf{Mediamill}} \\
\midrule
 & \textbf{Soft} & \textbf{Hard} &  &  &  & \textbf{Soft} & \textbf{Hard} &  &  &  \\
\midrule
Micro-F1 ($\uparrow$) & 0.22$\pm$0.06 & \textbf{0.23$\pm$0.04} & \multicolumn{3}{c}{} & 0.61$\pm$0.04 & \textbf{0.63$\pm$0.04} & \multicolumn{3}{c}{} \\
CE ($\downarrow$)     & 0.26$\pm$0.01 & \textbf{0.25$\pm$0.01} & \multicolumn{3}{c}{} & \textbf{0.14$\pm$0.01} & \textbf{0.14$\pm$0.01} & \multicolumn{3}{c}{} \\
RL ($\downarrow$)     & \textbf{0.17$\pm$0.01} & \textbf{0.17$\pm$0.01} & \multicolumn{3}{c}{} & 0.08$\pm$0.01 & \textbf{0.07$\pm$0.01} & \multicolumn{3}{c}{} \\
\bottomrule
\end{tabular}
\vspace{-1.2mm}
\end{table}

\section{Conclusion}
In this paper we addressed the partially multi-label feature selection problem, where unknown true labels need to be disambiguated, and errors may influence the feature selection process. From the perspective of POMDP, we propose a unified two-stage policy learning framework: Stage 1 transforms the candidate-constrained disambiguation problem into a horizon-1 POMDP to generate hard pseudo-labels; Stage 2 performs budget feature selection, which is a  decision process optimized by policy gradients, thus obtaining interpretable feature ranking under weak supervision.

Our analysis establishes the equivalence between the disambiguation risk and reward of the Stage 1 formula, providing an optimization guarantee for the Stage 1 under standard smoothing conditions. Furthermore, we derive an excess risk decomposition for downstream learning, separating the pseudo-label quality term from the usual generalization term. Experiments on the PML dataset validate the superiority of our method. In the future, we will further investigate the combination of reinforcement learning with partially labeled data scenarios, such as variance-reducing reinforcement learning training and extending disambiguation methods beyond hard first-stage disambiguation.

\section*{Impact Statement}
This paper presents work whose goal is to advance the field of Machine Learning. There are many potential societal consequences of our work, none which we feel must be specifically highlighted here.


\bibliography{example_paper}

@inproceedings{xie2018partial,
  title={Partial multi-label learning},
  author={Xie, Ming-Kun and Huang, Sheng-Jun},
  booktitle={Proceedings of the AAAI conference on artificial intelligence},
  volume={32},
  number={1},
  year={2018}
}

@inproceedings{yang2024noisy,
  title={Noisy label removal for partial multi-label learning},
  author={Yang, Fuchao and Jia, Yuheng and Liu, Hui and Dong, Yongqiang and Hou, Junhui},
  booktitle={Proceedings of the 30th ACM SIGKDD Conference on Knowledge Discovery and Data Mining},
  pages={3724--3735},
  year={2024}
}

@article{chen2022representation,
  title={Representation learning from noisy user-tagged data for sentiment classification},
  author={Chen, Long and Wang, Fei and Yang, Ruijing and Xie, Fei and Wang, Wenjing and Xu, Cai and Zhao, Wei and Guan, Ziyu},
  journal={International Journal of Machine Learning and Cybernetics},
  volume={13},
  number={12},
  pages={3727--3742},
  year={2022},
  publisher={Springer}
}

@article{murshed2022enhancing,
  title={Enhancing big social media data quality for use in short-text topic modeling},
  author={Murshed, Belal Abdullah Hezam and Abawajy, Jemal and Mallappa, Suresha and Saif, Mufeed Ahmed Naji and Al-Ghuribi, Sumaia Mohammed and Ghanem, Fahd A},
  journal={Ieee Access},
  volume={10},
  pages={105328--105351},
  year={2022},
  publisher={IEEE}
}

@inproceedings{tejero2023full,
  title={Full or Weak annotations? An adaptive strategy for budget-constrained annotation campaigns},
  author={Tejero, Javier Gamazo and Zinkernagel, Martin S and Wolf, Sebastian and Sznitman, Raphael and M{\'a}rquez-Neila, Pablo},
  booktitle={Proceedings of the IEEE/CVF Conference on Computer Vision and Pattern Recognition},
  pages={11381--11391},
  year={2023}
}

@inproceedings{mazzamuto2022weakly,
  title={Weakly supervised attended object detection using gaze data as annotations},
  author={Mazzamuto, Michele and Ragusa, Francesco and Furnari, Antonino and Signorello, Giovanni and Farinella, Giovanni Maria},
  booktitle={International Conference on Image Analysis and Processing},
  pages={263--274},
  year={2022},
  organization={Springer}
}

@article{mongardi2024biologically,
  title={Biologically weighted LASSO: enhancing functional interpretability in gene expression data analysis},
  author={Mongardi, Sofia and Cascianelli, Silvia and Masseroli, Marco},
  journal={Bioinformatics},
  volume={40},
  number={10},
  pages={btae605},
  year={2024},
  publisher={Oxford University Press}
}

@article{smarandache2024soft,
  title={Soft sets extensions used in bioinformatics},
  author={Smarandache, Florentin and Gifu, Daniela},
  journal={Procedia Computer Science},
  volume={246},
  pages={2185--2193},
  year={2024},
  publisher={Elsevier}
}

@article{hang2023partial,
  title={Partial multi-label learning with probabilistic graphical disambiguation},
  author={Hang, Jun-Yi and Zhang, Min-Ling},
  journal={Advances in Neural Information Processing Systems},
  volume={36},
  pages={1339--1351},
  year={2023}
}

@article{zhong2024negative,
  title={Negative label and noise information guided disambiguation for partial multi-label learning},
  author={Zhong, Jingyu and Shang, Ronghua and Zhao, Feng and Zhang, Weitong and Xu, Songhua},
  journal={IEEE Transactions on Multimedia},
  volume={26},
  pages={9920--9935},
  year={2024},
  publisher={IEEE}
}

@inproceedings{li2025calibrated,
  title={Calibrated disambiguation for partial multi-label learning},
  author={Li, Zhuoming and Jia, Yuheng and Yu, Mi and Miao, Zicong},
  booktitle={Proceedings of the AAAI Conference on Artificial Intelligence},
  volume={39},
  number={17},
  pages={18620--18628},
  year={2025}
}

@inproceedings{wang2022partial,
  title={Partial multi-label feature selection},
  author={Wang, Jing and Li, Peipei and Yu, Kui},
  booktitle={2022 International Joint Conference on Neural Networks (IJCNN)},
  pages={1--9},
  year={2022},
  organization={IEEE}
}

@inproceedings{pan2025reconsidering,
  title={Reconsidering Feature Structure Information and Latent Space Alignment in Partial Multi-label Feature Selection},
  author={Pan, Hanlin and Liu, Kunpeng and Gao, Wanfu},
  booktitle={Proceedings of the AAAI Conference on Artificial Intelligence},
  volume={39},
  number={19},
  pages={19786--19794},
  year={2025}
}

@article{zhang2020partial,
  title={Partial multi-label learning via credible label elicitation},
  author={Zhang, Min-Ling and Fang, Jun-Peng},
  journal={IEEE Transactions on Pattern Analysis and Machine Intelligence},
  volume={43},
  number={10},
  pages={3587--3599},
  year={2020},
  publisher={IEEE}
}

@inproceedings{xu2020partial,
  title={Partial multi-label learning with label distribution},
  author={Xu, Ning and Liu, Yun-Peng and Geng, Xin},
  booktitle={Proceedings of the AAAI conference on artificial intelligence},
  volume={34},
  number={04},
  pages={6510--6517},
  year={2020}
}

@article{xie2021partial,
  title={Partial multi-label learning with noisy label identification},
  author={Xie, Ming-Kun and Huang, Sheng-Jun},
  journal={IEEE Transactions on Pattern Analysis and Machine Intelligence},
  volume={44},
  number={7},
  pages={3676--3687},
  year={2021},
  publisher={IEEE}
}

@article{liang2024partial,
  title={Partial multi-label learning via exploiting instance and label correlations},
  author={Liang, Weichao and Gao, Guangliang and Chen, Lei and Wang, Youquan},
  journal={ACM Transactions on Knowledge Discovery from Data},
  volume={19},
  number={1},
  pages={1--22},
  year={2024},
  publisher={ACM New York, NY}
}

@inproceedings{cassandra1998survey,
  title={A survey of POMDP applications},
  author={Cassandra, Anthony R},
  booktitle={Working notes of AAAI 1998 fall symposium on planning with partially observable Markov decision processes},
  volume={1724},
  year={1998}
}

@article{arcieri2024pomdp,
  title={POMDP inference and robust solution via deep reinforcement learning: An application to railway optimal maintenance},
  author={Arcieri, Giacomo and Hoelzl, Cyprien and Schwery, Oliver and Straub, Daniel and Papakonstantinou, Konstantinos G and Chatzi, Eleni},
  journal={Machine Learning},
  volume={113},
  number={10},
  pages={7967--7995},
  year={2024},
  publisher={Springer}
}

@inproceedings{durand2019learning,
  title={Learning a deep convnet for multi-label classification with partial labels},
  author={Durand, Thibaut and Mehrasa, Nazanin and Mori, Greg},
  booktitle={Proceedings of the IEEE/CVF conference on computer vision and pattern recognition},
  pages={647--657},
  year={2019}
}

@article{wang2020semi,
  title={Semi-supervised partial label learning via confidence-rated margin maximization},
  author={Wang, Wei and Zhang, Min-Ling},
  journal={Advances in neural information processing systems},
  volume={33},
  pages={6982--6993},
  year={2020}
}

@article{xu2021progressive,
  title={Progressive enhancement of label distributions for partial multilabel learning},
  author={Xu, Ning and Liu, Yun-Peng and Zhang, Yan and Geng, Xin},
  journal={IEEE Transactions on Neural Networks and Learning Systems},
  volume={34},
  number={8},
  pages={4856--4867},
  year={2021},
  publisher={IEEE}
}

@article{gong2021understanding,
  title={Understanding partial multi-label learning via mutual information},
  author={Gong, Xiuwen and Yuan, Dong and Bao, Wei},
  journal={Advances in Neural Information Processing Systems},
  volume={34},
  pages={4147--4156},
  year={2021}
}

@inproceedings{duarte2021plm,
  title={Plm: Partial label masking for imbalanced multi-label classification},
  author={Duarte, Kevin and Rawat, Yogesh and Shah, Mubarak},
  booktitle={Proceedings of the IEEE/CVF Conference on Computer Vision and Pattern Recognition},
  pages={2739--2748},
  year={2021}
}

@inproceedings{feng2019partial,
  title={Partial label learning with self-guided retraining},
  author={Feng, Lei and An, Bo},
  booktitle={Proceedings of the AAAI conference on artificial intelligence},
  volume={33},
  number={01},
  pages={3542--3549},
  year={2019}
}

@inproceedings{lyu2020partial,
  title={Partial multi-label learning via probabilistic graph matching mechanism},
  author={Lyu, Gengyu and Feng, Songhe and Li, Yidong},
  booktitle={Proceedings of the 26th ACM SIGKDD International Conference on Knowledge Discovery \& Data Mining},
  pages={105--113},
  year={2020}
}

@article{sun2021global,
  title={Global-local label correlation for partial multi-label learning},
  author={Sun, Lijuan and Feng, Songhe and Liu, Jun and Lyu, Gengyu and Lang, Congyan},
  journal={IEEE Transactions on Multimedia},
  volume={24},
  pages={581--593},
  year={2021},
  publisher={IEEE}
}

@inproceedings{ijcai2025p576,
  title     = {Noise-Resistant Label Reconstruction Feature Selection for Partial Multi-Label Learning },
  author    = {Gao, Wanfu and Pan, Hanlin and Han, Qingqi and Liu, Kunpeng},
  booktitle = {Proceedings of the Thirty-Fourth International Joint Conference on
               Artificial Intelligence, {IJCAI-25}},
  publisher = {International Joint Conferences on Artificial Intelligence Organization},
  editor    = {James Kwok},
  pages     = {5172--5180},
  year      = {2025},
  month     = {8},
  note      = {Main Track},
  doi       = {10.24963/ijcai.2025/576},
  url       = {https://doi.org/10.24963/ijcai.2025/576},
}

@article{wu2025partial,
  title={Partial multi-label feature selection with feature noise},
  author={Wu, You and Li, Peipei and Zou, Yizhang},
  journal={Pattern Recognition},
  volume={162},
  pages={111310},
  year={2025},
  publisher={Elsevier}
}

@article{hao2025embedded,
  title={Embedded feature fusion for multi-view multi-label feature selection},
  author={Hao, Pingting and Gao, Wanfu and Hu, Liang},
  journal={Pattern Recognition},
  volume={157},
  pages={110888},
  year={2025},
  publisher={Elsevier}
}

@article{liu2021automated,
  title={Automated feature selection: A reinforcement learning perspective},
  author={Liu, Kunpeng and Fu, Yanjie and Wu, Le and Li, Xiaolin and Aggarwal, Charu and Xiong, Hui},
  journal={IEEE Transactions on Knowledge and Data Engineering},
  volume={35},
  number={3},
  pages={2272--2284},
  year={2021},
  publisher={IEEE}
}

@article{hao2023partial,
  title={Partial multi-label feature selection via subspace optimization},
  author={Hao, Pingting and Hu, Liang and Gao, Wanfu},
  journal={Information Sciences},
  volume={648},
  pages={119556},
  year={2023},
  publisher={Elsevier}
}

@inproceedings{yu2018feature,
  title={Feature-induced partial multi-label learning},
  author={Yu, Guoxian and Chen, Xia and Domeniconi, Carlotta and Wang, Jun and Li, Zhao and Zhang, Zili and Wu, Xindong},
  booktitle={2018 IEEE international conference on data mining (ICDM)},
  pages={1398--1403},
  year={2018},
  organization={IEEE}
}

@article{liu2023towards,
  title={Towards enabling binary decomposition for partial multi-label learning},
  author={Liu, Bing-Qing and Jia, Bin-Bin and Zhang, Min-Ling},
  journal={IEEE transactions on pattern analysis and machine intelligence},
  volume={45},
  number={11},
  pages={13203--13217},
  year={2023},
  publisher={IEEE}
}
\bibliographystyle{icml2026}

\newpage
\clearpage
\appendix
\onecolumn
\section{Pseudo Code of Algorithm}
\begin{algorithm}[!htbp]
  \caption{Two-Stage Horizon-1 POMDP Disambiguation + RL Feature Selection}
  \label{alg:two_stage_appendix}
  \begin{algorithmic}
    \STATE {\bfseries Input:} training set $\mathcal{D}=\{(x_i,C_i)\}_{i=1}^n$, label size $L$, feature dim $d$, budgets $k_{\mathrm{fs}}$, epochs $E_1,E_2$, 
     learning rates $\eta_{\mathrm{disc}},\eta_{\mathrm{pol}},\eta_{\mathrm{fs}}$, weights $\lambda_{\mathrm{struct}},\lambda_{\mathrm{fs}}$
    \STATE {\bfseries Output:} hard pseudo labels $\{\hat{Y}_i\}_{i=1}^n$, selected feature set $S$ (or ranking $\rho$)

    \STATE \textbf{Initialize} encoder $f_{\mathrm{enc}}$, disambiguation head $h_{\mathrm{disc}}$
    \STATE \textbf{Initialize} label policy $\pi_{\theta}$, feature policy $\pi_{\psi}$, baseline $b \leftarrow 0$

    \STATE {\bfseries Stage 1: Horizon-1 POMDP for hard label disambiguation}
    \FOR{$e=1$ {\bfseries to} $E_1$}
      \FOR{each minibatch $\mathcal{B}\subset \mathcal{D}$}
        \STATE Construct observation $O_i=(x_i,C_i)$ for each $(x_i,C_i)\in\mathcal{B}$
        \STATE Compute representation $\phi_i \leftarrow f_{\mathrm{enc}}(O_i)$
        \STATE Compute label probabilities $p_{i,j} \leftarrow \pi_{\theta}(j\mid O_i)$ for $j\in C_i$, and set $p_{i,j}\leftarrow 0$ for $j\notin C_i$
        \STATE Sample hard action $Z_i \sim \prod_{j=1}^{L}\mathrm{Bernoulli}(p_{i,j}) \ \ $ (a hard pseudo label vector)
        \STATE Predict $U_i \leftarrow h_{\mathrm{disc}}(\phi_i)$
        \STATE Compute disambiguation loss according to Formula (7)
        \STATE Update $(f_{\mathrm{enc}},h_{\mathrm{disc}})$ by gradient descent on $\mathcal{L}_{\mathrm{disc}}$
        \STATE Set reward $r \leftarrow -\mathcal{L}_{\mathrm{disc}}$
        \STATE Update label policy $\theta$ by policy gradient using $\log \pi_{\theta}(Z_i\mid O_i)$ and reward $r$
      \ENDFOR
    \ENDFOR

    \STATE {\bfseries Export hard pseudo labels}
    \FOR{each training instance $(x_i,C_i)\in\mathcal{D}$}
      \STATE Compute $p_{i,j} \leftarrow \pi_{\theta}(j\mid (x_i,C_i))$
      \STATE Set $\hat{Y}_{i,j}\leftarrow 1$ if $p_{i,j}\ge 0.5$ else $\hat{Y}_{i,j}\leftarrow 0$
    \ENDFOR

    \STATE {\bfseries Stage 2: RL-based sequential feature selection}
    \STATE Set horizon $T \leftarrow \min(k_{\mathrm{fs}}, d)$
    \FOR{$e=1$ {\bfseries to} $E_2$}
      \FOR{each minibatch $\mathcal{B}'=\{(x_i,\hat{Y}_i)\}$}
        \STATE Initialize mask $m_i \leftarrow \mathbf{0}_d$ and selected set $S_i \leftarrow \emptyset$ for each $i\in\mathcal{B}'$
        \STATE Initialize accumulator $\log\Pi \leftarrow 0$
        \FOR{$t=1$ {\bfseries to} $T$}
          \STATE Form partial observation $\tilde{x}_{i} \leftarrow m_i \odot x_i$
          \STATE Compute state embedding $h_{i} \leftarrow f_{\mathrm{enc}}(\tilde{x}_{i})$
          \STATE Sample feature action $a_i \sim \pi_{\psi}(\cdot \mid h_i)$ with already-selected dims masked out
          \STATE Update $S_i \leftarrow S_i \cup \{a_i\}$ and set $m_i[a_i]\leftarrow 1$
          \STATE Accumulate $\log\Pi \leftarrow \log\Pi + \frac{1}{|\mathcal{B}'|}\sum_{i\in\mathcal{B}'} \log \pi_{\psi}(a_i\mid h_i)$
        \ENDFOR

        \STATE Train predictor on selected features and compute supervised loss according to Formula (14)
       
        \STATE Set reward $r \leftarrow -\mathcal{L}_{\mathrm{sup}}$ and advantage $A \leftarrow r - b$
        \STATE Policy loss $\mathcal{L}_{\mathrm{pg}} \leftarrow -A\cdot \log\Pi$
        \STATE Total Stage-2 loss $\mathcal{L}_{\mathrm{fs}} \leftarrow \mathcal{L}_{\mathrm{sup}} + \lambda_{\mathrm{fs}}\mathcal{L}_{\mathrm{pg}}$
        \STATE Update $(f_{\mathrm{enc}},h_{\mathrm{disc}},\psi)$ by gradient descent on $\mathcal{L}_{\mathrm{fs}}$
        \STATE Update baseline $b \leftarrow (1-\beta)b + \beta r$
      \ENDFOR
    \ENDFOR

    \STATE Aggregate $\{S_i\}$ into global ranking $\rho$ (e.g., by selection frequency) and output $S$ (top-$k_{\mathrm{fs}}$)
  \end{algorithmic}
\end{algorithm}

\section{Proof of Theorem 5.2}

We  completely describe the relevant definitions in this section, and then prove (a) - (c) in sequence.

\subsection{Preliminaries}

PML data consists of triples \((X,C)\), where \(X\in\mathcal{X}\in\mathbb{R}^d\), \(Y\in\{0,1\}^L\), and \(C\) are candidate label sets, whose goal is to construct a function about \(X\). The surrogate disambiguation loss:
\begin{equation}
\ell_{\mathrm{disc}} : \mathcal{X} \times \{0,1\}^L \times 2^{[L]} \times {0,1}^L \to \mathbb{R}_+
\end{equation}
takes as input \((X,C,Z)\), where \(Z\in{0,1}^L\) satisfies \(Z_j=0\) whenever \(j\notin C\), and outputs a non-negative loss that measures the difference between \(Z\) and \(Y\). This loss function may contain structure terms that depend on \((X,C)\). We assume throughout that \(\ell{\mathrm{disc}}\) is measurable and integrable with respect to the relevant joint distributions.

A disambiguation rule is a mapping
\begin{equation}
g : \mathcal{X} \times 2^{[L]} \to \{0,1\}^L
\end{equation}
such that \(g(X,C)_j = 0\) whenever \(j\notin C\). We allow \(g\) to be random, i.e., its output may depend on an auxiliary source of randomness independent of \((X,Y,C)\), but we require that for each fixed \((x,c)\) the distribution of \(g(x,c)\) is supported on
\begin{equation}
\mathcal{A}(x, c)=\left\{z \in\{0,1\}^L: z_j=0 \text { whenever } j \notin c\right\} .
\end{equation}
Let \(\mathcal{G}\) denote the class of all such disambiguation rules. Given \(\ell_{\mathrm{disc}}\), the PML disambiguation risk of \(g\in\mathcal{G}\) is
\begin{equation}
\mathcal{R}_{\mathrm{PML}}(g)
= \mathbb{E}_{(X,C)}\big[\ell_{\mathrm{disc}}\big(g(X,C),Z\big)\big],
\end{equation}

The horizon-1 POMDP of Section~4.1 is defined as follows. The latent state is \(S=(X,Y,C)\), the observation is \(O=(X,C)\), the actions given \((X,C)\) are hard label vectors \(Z\in\mathcal{A}(X,C)\), and the one-step reward is
\begin{equation}
r(S,Z) = -\ell_{\mathrm{disc}}\big(g(X,C),C\big)\
\end{equation}
A stochastic policy \(\pi\) assigns a probability distribution \(\pi(\cdot\mid X,C)\) to each observation \(O=(X,C)\), defined on \(\mathcal{A}(X,C)\). Let \(\Pi\) denote the class of such policies. The joint distribution of \((X,C,Z)\) induced by \(\pi\) is the product of \(\mathbb{P}\) on \((X,C)\) and \(\pi(\cdot\mid X,C)\) on \(Z\). For \(\pi\in\Pi\), its expected return is:
\begin{equation}
J(\pi)
= \mathbb{E}{(S,O,Z)\sim\pi}[r(S,Z)]
= \mathbb{E}_{(X,C)}\big[r(S,Z)\big]
= -\mathbb{E}_{(X,C)}\big[\ell_{\mathrm{disc}}(X,C,Z)\big].
\end{equation}

\subsection{Proof of Theorem 5.2}

We now prove the three parts of Theorem 5.2.

\paragraph{Proof of (a).}
For any \(\pi\in\Pi\), \(\pi\) specifies the conditional distribution \(\pi(\cdot\mid x,c)\) on \(\mathcal{A}(x,c)\) for each \((x,c)\). Its disambiguation rule is \(g_\pi\):
\begin{equation}
g_\pi(x,c) := Z, \quad Z \sim \pi(\cdot\mid x,c).
\end{equation}
According to the construction, \(g_\pi(x,c)\in\mathcal{A}(x,c)\) is almost certainly true, therefore \(g_\pi\in\mathcal{G}\).

Consider the joint distribution of \((X,C,Z)\) obtained by first sampling \((X,C)\sim\mathbb{P}\) and then sampling \(Z\sim\pi(\cdot\mid X,C)\). Under this joint distribution, the risk is:
\begin{equation}
\mathcal{R}_{\mathrm{PML}}(g_\pi)
= \mathbb{E}_{(X,C)}\big[\ell_{\mathrm{disc}}(X,C, g(X,C))\big]
= \mathbb{E}_{(X,C)}\big[\ell_{\mathrm{disc}}\big(X,C,Z\big)\big]
,
\end{equation}
the second equation holds because the conditional distribution of \(Z\) given \((X,C)\) is \(\pi(\cdot\mid X,C)\), which is the same as the conditional distribution in POMDP. On the other hand, the reward is:
\begin{equation}
J(\pi)
= -\mathbb{E}_{(X,C)}\big[\ell_{\mathrm{disc}}(X,C,Z)\big].
\end{equation}
Hence
\begin{equation}
J(\pi) = -\mathcal{R}_{\mathrm{PML}}(g_\pi),
\end{equation}
which proves part (a).

\paragraph{Proof of (b).}
For any \(g\in\mathcal{G}\), let \(\mathbb{Q}{g(\cdot\mid x,c)}\) denote the distribution of \(Z = g(x,c)\) on \(\mathcal{A}(x,c)\) for every \((x,c)\), which is induced by the internal randomness of \(g\). Define the policy \(\pi_g\) as:
\begin{equation}
\pi_g(\cdot\mid x,c) := \mathbb{Q}_{g(\cdot\mid x,c)}.
\end{equation}
Then, for each \((X,C)\), the distribution of \(Z\sim\pi_g(\cdot\mid X,C)\) coincides with that of \(g(X,C)\). In particular,
\begin{equation}
g_{\pi_g}(X,C) := Z,\quad Z\sim\pi_g(\cdot\mid X,C)
\end{equation}
has the same distribution as \(g(X,C)\), so \(g_{\pi_g} = g\) almost surely.

With \((X,C,Z)\) sampled from \((X,C)\sim\mathbb{P}\) and \(Z\sim\pi_g(\cdot\mid X,C)\), we have:
\begin{equation}
\mathcal{R}_{\mathrm{PML}}(g)
= \mathbb{E}_{(X,C)}\big[\ell_{\mathrm{disc}}(X,C, g(X,C))\big]
= \mathbb{E}_{(X,C)}\big[\ell_{\mathrm{disc}}(X,C,Z)\big],
\end{equation}
and
\begin{equation}
J(\pi_g)
= -\mathbb{E}_{(X,C)}\big[\ell_{\mathrm{disc}}(X,C,Z)\big].
\end{equation}
Therefore
\begin{equation}
J(\pi_g) = -\mathcal{R}_{\mathrm{PML}}(g),
\end{equation}
which proves part (b).

\paragraph{Proof of (c).}
Define
\begin{equation}
\mathcal{R}_{\mathrm{PML}}^\star := \inf_{g\in\mathcal{G}} \mathcal{R}_{\mathrm{PML}}(g),
\qquad
J^\star := \sup_{\pi\in\Pi} J(\pi).
\end{equation}
By part (a), for any \(\pi\in\Pi\),
\begin{equation}
J(\pi) = -\mathcal{R}_{\mathrm{PML}}(g_\pi),
\end{equation}
so
\begin{equation}
J^\star
= \sup_{\pi\in\Pi} J(\pi)
= \sup_{\pi\in\Pi} \big(-\mathcal{R}_{\mathrm{PML}}(g_\pi)\big)
\le -\inf_{g\in\mathcal{G}} \mathcal{R}_{\mathrm{PML}}(g)
= -\mathcal{R}_{\mathrm{PML}}^\star.
\end{equation}

Conversely, by part (a), for any \(g\in\mathcal{G}\),
\begin{equation}
J(\pi_g) = -\mathcal{R}_{\mathrm{PML}}(g),
\end{equation}
hence
\begin{equation}
J^\star
= \sup_{\pi\in\Pi} J(\pi)
\ge \sup_{g\in\mathcal{G}} J(\pi_g)
= \sup_{g\in\mathcal{G}} \big(-\mathcal{R}_{\mathrm{PML}}(g)\big)
= -\inf_{g\in\mathcal{G}} \mathcal{R}_{\mathrm{PML}}(g)
= -\mathcal{R}_{\mathrm{PML}}^\star.
\end{equation}
Combining the two inequalities we get $J^\star = -\mathcal{R}_{\mathrm{PML}}^\star$.

Finally, let \(g^\star \in \arg\min_{g\in\mathcal{G}} \mathcal{R}_{\mathrm{PML}}(g)\) be Bayes-optimal. By part (b),
\begin{equation}
J(\pi_{g^\star}) = -\mathcal{R}_{\mathrm{PML}}(g^\star)
= -\mathcal{R}_{\mathrm{PML}}^\star
= J^\star,
\end{equation}
so $\pi_{g^\star}\in\arg\max_{\pi\in\Pi} J(\pi)$ is an optimal policy. Conversely, if $\pi^\star\in\arg\max_{\pi\in\Pi} J(\pi)$ is optimal, then by part (a),
\begin{equation}
\mathcal{R}_{\mathrm{PML}}(g_{\pi^\star})
= -J(\pi^\star)
= -J^\star
= \mathcal{R}_{\mathrm{PML}}^\star,
\end{equation}
Therefore, \(g_{\pi^\star}\in\arg\min_{g\in\mathcal{G}} \mathcal{R}_{\mathrm{PML}}(g)\) is Bayes-optimal. This establishes the correspondence between the optimal policy and the Bayes-optimal disambiguation rule, completing the proof of Theorem 5.2.

\section{Proof of Theorem 5.3}

In this appendix, we provide a convergence analysis of the Stage-1 policy gradient update under an idealized setting, where the encoder and discriminator are fixed and only the policy head parameters are optimized.

\subsection{Surrogate objective and parameterization}

Recall that each policy \(\pi_\theta\) generates a disambiguation rule \(g_\theta\), and its partial multi-label disambiguation risk can be represented as:
\begin{equation}
\mathcal{R}_{\mathrm{PML}}(\theta)
= \mathbb{E}_{(X,C),Z\sim\pi_\theta(\cdot\mid X,C)}\big[\ell_{\mathrm{disc}}(X,C,Z)\big].
\end{equation}
In the stage 1, we parameterize \(\pi_\theta(\cdot\mid X,C)\)using a policy head that generates logits from a fixed representation \(h=f_{\mathrm{enc}}(X,C)\), and then perform element-wise Bernoulli sampling on the candidate labels:
\begin{equation}
Z \sim \pi_\theta(\cdot\mid X,C).
\end{equation}
The exact form of \(\pi_\theta\) is not required in the proof; we only assume that it is differentiable in \(\theta\) and that the corresponding probability mass function \(\pi_\theta(Z\mid X,C)\) is strictly positive for all admissible actions \(Z\).

For convenience, we define
\begin{equation}
F(\theta) := \mathcal{R}_{\mathrm{PML}}(\theta)
= \mathbb{E}\big[\ell_{\mathrm{disc}}(X,C,Z)\big],
\end{equation}

and analyze stochastic gradient descent on \(F\).

\subsection{Policy-gradient identity and bounded variance}

We first prove that the reinforcement estimator used in the first stage is an unbiased estimator of the gradient \(\nabla_\theta F(\theta)\), and then give sufficient conditions for the bounded second moments.

\begin{lemma}[Policy-gradient identity]
Suppose that \(\ell_{\mathrm{disc}}(X,C,Z)\) is integrable, and its dependence on \(\theta\) is manifested only through \(\pi_\theta(Z\mid X,C)\). We further assume that differentiation under expectation is reasonable. Then for any \(\theta\):
\begin{equation}
\nabla_\theta F(\theta)
= \mathbb{E}\big[\ell_{\mathrm{disc}}(X,C,Z)\nabla_\theta \log \pi_\theta(Z\mid X,C)\big],
\end{equation}
where the expectation is taken over \((X,C)\sim\mathbb{P}\) and \(Z\sim\pi_\theta(\cdot\mid X,C)\).
\end{lemma}

\begin{proof}
By definition,
\begin{equation}
F(\theta)
= \mathbb{E}_{(X,C)}\Big[\mathbb{E}_{Z\sim\pi_\theta(\cdot\mid X,C)}\big[\ell_{\mathrm{disc}}(X,C,Z)\big]\Big].
\end{equation}

By changing the order of integration and differentiation, we obtain:
\begin{equation}
\nabla_\theta F(\theta)
=\mathbb{E}_{(X,C)}\left[
\sum_{z\in \mathcal{Z}(C)}
\tilde{\ell}_{\mathrm{disc}}(X,C,z)\ \nabla_\theta \pi_\theta(z\mid X,C)
\right]
=\mathbb{E}_{(X,C),Z\sim\pi_\theta}\left[
\tilde{\ell}_{\mathrm{disc}}(X,C,Z;)\ \nabla_\theta \log \pi_\theta(Z\mid X,C)
\right].
\end{equation}

where the sum is over admissible actions \(z\). For each fixed \((X,C)\), we have
\begin{equation}
\nabla_\theta
\int \ell_{\mathrm{disc}}(X,C,z)\pi_\theta(z\mid X,C)dz
= \int \ell_{\mathrm{disc}}(X,C,z)\nabla_\theta \pi_\theta(z\mid X,C)dz.
\end{equation}
Using the log-derivative trick,
\begin{equation}
\nabla_\theta \pi_\theta(z\mid X,C)
= \pi_\theta(z\mid X,C)\nabla_\theta \log \pi_\theta(z\mid X,C),
\end{equation}
so
\begin{equation}
\nabla_\theta
\int \ell_{\mathrm{disc}}(X,C,z)\pi_\theta(z\mid X,C)dz
= \int \ell_{\mathrm{disc}}(X,C,z)\pi_\theta(z\mid X,C)
\nabla_\theta \log \pi_\theta(z\mid X,C)dz.
\end{equation}
Treating this sum as the conditional expectation of \(Z\sim\pi_\theta(\cdot\mid X,C)\), we can obtain:
\begin{equation}
\nabla_\theta F(\theta)
= \mathbb{E}\big[\ell_{\mathrm{disc}}(X,C,Z)\nabla_\theta \log \pi_\theta(Z\mid X,C)\big],
\end{equation}
as claimed.
\end{proof}

Therefore if in the t-th iteration, we draw \((X_t,C_t)\) from the PML data distribution and sample \(Z_t\sim\pi_{\theta_t}(\cdot\mid X_t,C_t)\), then the  stochastic vector:
\begin{equation}
g_t
= \ell_{\mathrm{disc}}(X_t,C_t,Z_t)\nabla_\theta \log \pi_\theta(Z_t\mid X_t,C_t)\Big|_{\theta=\theta_t}
\end{equation}
satisfies
\begin{equation}
\mathbb{E}[g_t \mid \theta_t] = \nabla_\theta F(\theta_t),
\end{equation}
i.e., $g_t$ is an unbiased estimator of the true gradient.

\begin{lemma}[Bounded second moment]
Suppose that for some \(M>0\), it is almost certain that \(\vert\ell_{\mathrm{disc}}(X,C,Z)\vert\le M\), and for all \(\theta\in\Theta\), it is almost certain that \(\Vert\nabla_\theta \log \pi_\theta(Z\mid X,C)\Vert\le L_\pi\). Then there exists a constant \(G>0\) such that:
\begin{equation}
\mathbb{E}\big[\Vert g_t\Vert^2 \mid \theta_t\big] \le G^2
\quad\text{for all } t.
\end{equation}
\end{lemma}

\begin{proof}
According to the definition of \(g_t\):
\begin{equation}
\Vert g_t\Vert
= \big\vert\ell_{\mathrm{disc}}(X_t,C_t,Z_t)\big\vert
\Big\Vert\nabla_\theta \log \pi_\theta(Z_t\mid X_t,C_t)\Big\vert_{\theta=\theta_t}\Big\Vert
\le M L_\pi.
\end{equation}
Hence \(\Vert g_t\Vert^2 \le M^2 L_\pi^2\) is almost certainly true, and by taking the conditional expectation, we can obtain:
\begin{equation}
\mathbb{E}\big[\Vert g_t\Vert^2 \mid \theta_t\big] \le M^2 L_\pi^2.
\end{equation}
The result follows by setting $G = M L_\pi$.
\end{proof}

This establishes the unbiased and bounded variance conditions used in the convergence analysis of Stage 1 .

\subsection{Stochastic gradient descent convergence}

Now we will use the standard proof of the stochastic gradient method on smooth targets to prove Theorem 5.3.

\begin{proof}[Proof of Theorem 5.3]
By assumption (a), \(F(\theta)\) is bounded below and has Lipschitz-continuous gradient on \(\Theta\), i.e., there exists \(L>0\) such that
\begin{equation}
\big\Vert\nabla_\theta F(\theta) - \nabla_\theta F(\theta')\big\Vert
\le L \Vert\theta - \theta'\Vert
\quad\text{for all } \theta,\theta'\in\Theta.
\end{equation}
Let \(\theta_{t+1} = \theta_t - \alpha_t g_t\) be the stochastic gradient descent update function, where \(\alpha_t>0\) and \(g_t\) is as shown above. Using the standard descent lemma for smooth functions, for any \(\theta\) and any \(\Delta\theta\), we have:
\begin{equation}
F(\theta + \Delta\theta)
\le F(\theta) + \nabla_\theta F(\theta)^{\top} \Delta\theta + \frac{L}{2}\Vert\Delta\theta\Vert^2.
\end{equation}
Applying this with \(\theta=\theta_t\) and \(\Delta\theta = -\alpha_t g_t\) yields
\begin{equation}
F(\theta_{t+1})
\le F(\theta_t) - \alpha_t \nabla_\theta F(\theta_t)^{\top} g_t+
\frac{L}{2}\alpha_t^2 \Vert g_t\Vert^2.
\end{equation}

Taking the conditional expectation of the \(\sigma\) algebra generated from the history up to time \(t\), and using the preceding lemma C.1 and C.2, we obtain:
\begin{equation}
\mathbb{E}\big[F(\theta_{t+1}) \mid \theta_t\big]
\le F(\theta_t) - \alpha_t \big\Vert\nabla_\theta F(\theta_t)\big\Vert^2+
\frac{L}{2}\alpha_t^2 G^2.
\end{equation}
Taking full expectations and rearranging gives
\begin{equation}
\mathbb{E}\big[F(\theta_{t+1})\big]
\le \mathbb{E}\big[F(\theta_t)\big]-
\alpha_t\mathbb{E}\big[\Vert\nabla_\theta F(\theta_t)\Vert^2\big]+
\frac{L}{2}\alpha_t^2 G^2.
\end{equation}

Summing the inequality from \(t=0\) to \(T-1\) and combining the left-hand sides, we get:
\begin{equation}
\mathbb{E}\big[F(\theta_T)\big]
\le \mathbb{E}\big[F(\theta_0)\big]
-
\sum_{t=0}^{T-1} \alpha_t\mathbb{E}\big[\Vert\nabla_\theta F(\theta_t)\Vert^2\big]
+
\frac{L}{2}G^2 \sum_{t=0}^{T-1} \alpha_t^2.
\end{equation}
Since \(F\) is bounded below, the left-hand side is uniformly bounded below \(T\). Using assumption (c), i.e., \(\sum_{t=0}^{\infty} \alpha_t^2<\infty\), we can let \(T\to\infty\) and derive:
\begin{equation}
\sum_{t=0}^{\infty} \alpha_t\mathbb{E}\big[\Vert\nabla_\theta F(\theta_t)\Vert^2\big]
< \infty.
\end{equation}
Combined with \(\sum_{t=0}^{\infty} \alpha_t = \infty\), this implies
\begin{equation}
\liminf_{t\to\infty} \mathbb{E}\big[\Vert\nabla_\theta F(\theta_t)\Vert^2\big] = 0.
\end{equation}
Which is a standard argument that \(F(\theta_t)\) is almost necessarily convergent and
\begin{equation}
\lim_{t\to\infty} \Vert\nabla_\theta F(\theta_t)\Vert = 0
\quad\text{almost surely}.
\end{equation}
Therefore every limit point of \({\theta_t}\) is a first-order stationary point of $F$. Then the proof of Theorem 5.3 can be completed by recalling the definition of $F(\theta)$ in the main text.
\end{proof}
\section{Proof of Theorem 5.5}
In this appendix we formalize the excess risk decomposition stated in Theorem 5.5. Our method is a two-stage procedure where Stage 1 learns a disambiguation rule $g$ and Stage 2 trains a downstream
predictor (and/or feature-selection module) using pseudo labels produced by $g$.
While this end-to-end reuse of training data is standard in practice, it introduces statistical dependence that
complicates a clean generalization analysis.
To keep the proof concise and to enable standard uniform convergence arguments, we adopt a conventional
sample-splitting (or cross-fitting) protocol in the analysis below.

\begin{assumption}[Stage Independence Based on Sample Splitting] For theoretical analysis, we split the training data into two independent subsets $\mathcal{D}_1$ and $\mathcal{D}_2$. Stage 1 uses only $\mathcal{D}_1$ to learn the disambiguation rule $g$. Stage 2 uses only $\mathcal{D}_2$ and the pseudo-labels generated by $g$ to train the classifier/feature selection process. Therefore, given $g$, the samples in $\mathcal{D}_2$ remain independent and identically distributed.

\end{assumption}

\subsection{Preliminaries}
In this appendix, we conditionally use the first-stage rule $g$ learned from $\mathcal{D}_1$. According to the assumption D.1, the second-stage samples $\mathcal{D}_2=\{(X_i,Y_i,C_i)\}_{i=1}^n$ are i.i.d (conditional on $g$). We revisit the population risk and empirical risk defined in the main text. For the predictor variable $f\in\mathcal{F}$, the true classification risk under $\ell_{\mathrm{cls}}$ is:
\begin{equation}
\mathcal{R}_{\mathrm{cls}}(f)
= \mathbb{E}\big[\ell_{\mathrm{cls}}(f(X),Y)\big].
\end{equation}
Given a disambiguation rule \(g\) and pseudo-labels \(Z=g(X,C)\), the corresponding pseudo-label risk is:
\begin{equation}
\mathcal{R}_{\mathrm{cls}}^{(g)}(f)
= \mathbb{E}\big[\ell_{\mathrm{cls}}(f(X),Z)\big].
\end{equation}
For a finite sample \({(X_i,Y_i,C_i)}{i=1}^n\) and \(Z_i = g(X_i,C_i)\), the empirical pseudo-label risk is:
\begin{equation}
\widehat{\mathcal{R}}_{\mathrm{cls}}^{(g)}(f)
= \frac{1}{n}\sum_{i=1}^n \ell_{\mathrm{cls}}(f(X_i),Z_i).
\end{equation}

We measure the average difference between real and fake labels using the following method:
\begin{equation}
\varepsilon_{\mathrm{pseudo}}(g)
= \mathbb{E}\big[d_{\mathrm{H}}(Y,g(X,C))\big],
\end{equation}
Where \(d_{\mathrm{H}}(y,z)\) represents the Hamming distance over \(\{0,1\}^L\) scaled to the interval \([0,1]\). Assumption 5.4 states that \(\ell_{\mathrm{cls}}\) is bounded over \([0,1]\) and that its label parameters satisfy the Lipschitz condition with respect to \(d_{\mathrm{H}}\).

\subsection{Deviation between true risk and pseudo-label risk}

We first show that, for any fixed predictor \(f\in\mathcal{F}\), the difference between its true risk and its pseudo-label risk is controlled by the average pseudo-label noise level.

\begin{lemma}[Label-noise deviation bound]
Under Assumption 5.4, for any \(f\in\mathcal{F}\),
\begin{equation}
\big\vert\mathcal{R}_{\mathrm{cls}}(f) - \mathcal{R}_{\mathrm{cls}}^{(g)}(f)\big\vert
\le L_{\mathrm{cls}}\varepsilon_{\mathrm{pseudo}}(g).
\end{equation}
\end{lemma}

\begin{proof}
The right half of Formula (89) can be written as:
\begin{equation}
\mathcal{R}_{\mathrm{cls}}(f) - \mathcal{R}_{\mathrm{cls}}^{(g)}(f)
= \mathbb{E}\big[\ell_{\mathrm{cls}}(f(X),Y) - \ell_{\mathrm{cls}}(f(X),Z)\big].
\end{equation}
Taking its absolute value and applying the Lipschitz property of \(\ell_{\mathrm{cls}}\), we can get:
\begin{equation}
\big\vert\mathcal{R}_{\mathrm{cls}}(f) - \mathcal{R}_{\mathrm{cls}}^{(g)}(f)\big\vert
\le \mathbb{E}\big[\big\vert\ell_{\mathrm{cls}}(f(X),Y) - \ell_{\mathrm{cls}}(f(X),Z)\big\vert\big]
\le L_{\mathrm{cls}}\mathbb{E}\big[d_{\mathrm{H}}(Y,Z)\big]
= L_{\mathrm{cls}}\varepsilon_{\mathrm{pseudo}}(g),
\end{equation}
which proves the claim.
\end{proof}

Specifically, this shows that the mappings \(f\mapsto \mathcal{R}{\mathrm{cls}}(f)\) and \(f\mapsto \mathcal{R}{\mathrm{cls}}^{(g)}(f)\) are consistently close on \(\mathcal{F}\), with the deviation depending only on the pseudo-label error of \(g\).

\subsection{Uniform convergence for empirical pseudo-label risk}

We next recall a standard uniform convergence bound that controls the deviation between the overall pseudo-label risk and its empirical corresponding value. Fix $g$ (conditional on Stage 1), let:
\begin{equation}
\mathcal{G}:=\left\{(x, z) \mapsto \ell_{\mathrm{cls}}(f(x), z): f \in \mathcal{F}\right\},
\end{equation}
which represents the class of loss functions induced by \(\mathcal{F}\) on the pair \((X,Z)\).Let $\mathfrak{R}_n(\mathcal{G})$ denote the empirical Rademacher complexity of $\mathcal{G}$ on the second-stage samples $\{(X_i,Z_i)\}_{i=1}^n$. Since $\ell_{\mathrm{cls}}$ is constrained by $M$ (assumption 5.4), a standard Rademacher-based bound implies that for any $\delta\in(0,1)$, with probability at least $1-\delta$ over the draw of $\mathcal{D}_2$ (conditional on $g$), we have:
\begin{equation}
\sup_{f\in\mathcal{F}}
\big\vert\mathcal{R}_{\mathrm{cls}}^{(g)}(f) - \widehat{\mathcal{R}}_{\mathrm{cls}}^{(g)}(f)\big\vert
\le 2\mathfrak{R}_n(\mathcal{G})+3M\sqrt{\frac{\log(2/\delta)}{2n}}.
\end{equation}
For compactness, we define:
\begin{equation}
\mathrm{Gen}_n(\mathcal{F}\delta)
:= 2\mathfrak{R}_n(\mathcal{G})
+
3M\sqrt{\frac{\log(2/\delta)}{2n}},
\end{equation}
so that, with probability at least \(1-\delta\),
\begin{equation}
\sup_{f\in\mathcal{F}}
\big\vert\mathcal{R}_{\mathrm{cls}}^{(g)}(f) - \widehat{\mathcal{R}}_{\mathrm{cls}}^{(g)}(f)\big\vert
\le \mathrm{Gen}_n(\mathcal{F},\delta).
\end{equation}

\subsection{Proof of Theorem 5.5}

We now combine the previous bounds to prove the excess risk decomposition.

\begin{proof}[Proof of Theorem 5.5]
Let \(\delta\in(0,1)\) be fixed. On the events where the above uniform deviation bound holds, for all \(f\in\mathcal{F}\), we have:
\begin{equation}
\mathcal{R}_{\mathrm{cls}}^{(g)}(f)
\le \widehat{\mathcal{R}}_{\mathrm{cls}}^{(g)}(f)+
\mathrm{Gen}_n(\mathcal{F},\delta),
\end{equation}
and
\begin{equation}
\widehat{\mathcal{R}}_{\mathrm{cls}}^{(g)}(f)
\le \mathcal{R}_{\mathrm{cls}}^{(g)}(f)+
\mathrm{Gen}_n(\mathcal{F},\delta).
\end{equation}

Let \(\hat{f}\in\arg\min_{f\in\mathcal{F}}\widehat{\mathcal{R}}{\mathrm{cls}}^{(g)}(f)\) be the empirical pseudo-label risk minimizer, and let \(f^\star\in\arg\min{f\in\mathcal{F}}\mathcal{R}{\mathrm{cls}}(f)\) be the minimizer of the true risk on \(\mathcal{F}\). We first compare the pseudo-label risks:
\begin{equation}
\mathcal{R}_{\mathrm{cls}}^{(g)}(\hat{f})
\le \widehat{\mathcal{R}}_{\mathrm{cls}}^{(g)}(\hat{f})+\mathrm{Gen}_n(\mathcal{F},\delta)
\le \widehat{\mathcal{R}}_{\mathrm{cls}}^{(g)}(f^\star)+\mathrm{Gen}_n(\mathcal{F},\delta)
\le \mathcal{R}_{\mathrm{cls}}^{(g)}(f^\star)+2\mathrm{Gen}_n(\mathcal{F},\delta).
\end{equation}

Next, using the label-noise deviation bound, we relate true risks and pseudo-label risks:
\begin{equation}
\mathcal{R}_{\mathrm{cls}}(\hat{f})
\le \mathcal{R}_{\mathrm{cls}}^{(g)}(\hat{f})+L_{\mathrm{cls}}\varepsilon_{\mathrm{pseudo}}(g),
\end{equation}
and
\begin{equation}
\mathcal{R}_{\mathrm{cls}}^{(g)}(f^\star)
\le \mathcal{R}_{\mathrm{cls}}(f^\star)+L_{\mathrm{cls}}\varepsilon_{\mathrm{pseudo}}(g).
\end{equation}
Combining these two inequalities, we obtain:
\begin{equation}
\begin{aligned}
\mathcal{R}_{\mathrm{cls}}(\hat{f})
&\le \mathcal{R}_{\mathrm{cls}}^{(g)}(\hat{f})
+ L_{\mathrm{cls}}\varepsilon_{\mathrm{pseudo}}(g) \\
&\le \mathcal{R}_{\mathrm{cls}}^{(g)}(f^\star)
+ 2\mathrm{Gen}_n(\mathcal{F},\delta)
+ L_{\mathrm{cls}}\varepsilon_{\mathrm{pseudo}}(g) \\
&\le \mathcal{R}_{\mathrm{cls}}(f^\star)
+ 2 L_{\mathrm{cls}}\varepsilon_{\mathrm{pseudo}}(g)
+ 2\mathrm{Gen}_n(\mathcal{F},\delta).
\end{aligned}
\end{equation}

By rearranging the terms, we obtain:
\begin{equation}
\mathcal{R}_{\mathrm{cls}}(\hat{f}) - \mathcal{R}_{\mathrm{cls}}(f^\star)
\le 2L_{\mathrm{cls}}\varepsilon_{\mathrm{pseudo}}(g)
+2\mathrm{Gen}_n(\mathcal{F},\delta),
\end{equation}
which holds with a probability of at least \(1 - \delta\) over the draw of $\mathcal{D}_2$. This is precisely the bound described in Theorem 5.5.
\end{proof}

\section{Details of nine benchmark datasets.}
\begin{table}[H]
\footnotesize 
\label{tab:my-table1}
\centering
\caption{Details of nine benchmark datasets.}
\begin{tabular}{cp{1.1cm}p{1.2cm}p{1.2cm}p{1.0cm}}
\hline
\textbf{Datasets} & \textbf{Domain} & \#\textbf{Instances} & \#\textbf{Features} & \#\textbf{Labels} \\
\hline
Bibtex & Text & 1953 & 1836 & 15 \\
Birds & Audio & 645 & 260 & 19 \\
CHD\_49 & Medicine & 555 & 49 & 6 \\
Chess & Image	& 258 & 585 & 15 \\
HumanPseAAC & Biology & 3106 & 40 & 14 \\
Mediamill & Image & 8276 & 120 & 15 \\
PlantPseAAC & Biology & 978 & 440 & 12 \\
Slashdot & Text & 2701 & 1079 & 15 \\
Yeast & Biology & 2417 & 103 & 14 \\

\hline
\end{tabular}
\end{table}
\section{Experimental results  on  Hamming Loss (HL) and Coverage Error (CE)}
\begin{table*}[htbp]
    \centering
  
    \label{tab:full_results}
    
    \scriptsize
    \setlength{\tabcolsep}{1.8pt}
    \renewcommand{\arraystretch}{0.82}
    \caption{Experimental results (mean ± std.) on 9 datasets showing Hamming Loss (HL $\downarrow$) and Coverage Error (CE $\downarrow$).}
    \begin{tabular}{@{}l*{10}{c}@{}}
        \toprule
        \textbf{Datasets} & \textbf{\textsc{POMDP-FS}} & \textbf{\textsc{PML-FSLA}} & \textbf{\textsc{PML-FSMIR}} & \textbf{\textsc{PML-FSSO}} & \textbf{\textsc{fPML}} & \textbf{\textsc{PML-LD}} & \textbf{\textsc{PAMB}} & \textbf{\textsc{PML-VLS}} & \textbf{\textsc{PML-MAP}} \\ 
        \midrule
        \multicolumn{10}{c}{\textbf{HL} ($\downarrow$)} \\ 
        \midrule
       Bibtex      & 0.06$\pm$0.01 & 0.90$\pm$0.00 & 0.06$\pm$0.00 & 0.19$\pm$0.00 & \textbf{0.04$\pm$0.01} & 0.06$\pm$0.01 & 0.07$\pm$0.01 & 0.08$\pm$0.01 & 0.11$\pm$0.02 \\
        Birds       & \textbf{0.06$\pm$0.00} & 0.15$\pm$0.01 & 0.32$\pm$0.09 & 0.12$\pm$0.00 & 0.13$\pm$0.00 & 0.13$\pm$0.01 & 0.20$\pm$0.02 & 0.11$\pm$0.00 & 0.10$\pm$0.01 \\
        CHD\_49     & \textbf{0.31$\pm$0.00} & 0.42$\pm$0.00 & 0.40$\pm$0.02 & 0.42$\pm$0.02 & 0.42$\pm$0.00 & 0.44$\pm$0.01 & 0.39$\pm$0.01 & 0.39$\pm$0.03 & 0.38$\pm$0.03 \\
        Chess       & \textbf{0.04$\pm$0.00} & 0.07$\pm$0.00 & 0.06$\pm$0.00 & 0.06$\pm$0.00 & 0.07$\pm$0.00 & 0.10$\pm$0.02 & 0.10$\pm$0.02 & 0.05$\pm$0.01 & 0.05$\pm$0.01 \\
        HumanPseAAC & 0.09$\pm$0.00 & 0.58$\pm$0.02 & 0.11$\pm$0.01 & 0.08$\pm$0.00 & \textbf{0.08$\pm$0.00} & 0.14$\pm$0.01 & 0.11$\pm$0 & 0.09$\pm$0.00 & 0.10$\pm$0.01 \\
        Mediamill   & \textbf{0.05$\pm$0.00} & 0.29$\pm$0.12 & 0.07$\pm$0.00 &  0.07$\pm$0.00 & 0.07$\pm$0.00 & 0.07$\pm$0.00 & 0.07$\pm$0.00 & 0.08$\pm$0.01 & 0.11$\pm$0.02 \\
        PlantPseAAC & \textbf{0.09$\pm$0.00} & 0.89$\pm$0.05 & 0.09$\pm$0.00 & 0.09$\pm$0.00 & 0.09$\pm$0.00 & 0.19$\pm$0.01 & 0.17$\pm$0.01 & 0.10$\pm$0.01 & 0.11$\pm$0.01 \\
        Slashdot    & \textbf{0.02$\pm$0.00} & 0.07$\pm$0.00 & 0.03$\pm$0.01 & 0.03$\pm$0.01 & 0.07$\pm$0.00 & 0.06$\pm$0.01 & 0.12$\pm$0.01 & 0.06$\pm$0.01 & 0.06$\pm$0.01 \\
        Yeast       & \textbf{0.22$\pm$0.01} & 0.56$\pm$0.04 & 0.39$\pm$0.09 & 0.30$\pm$0.00 & 0.30$\pm$0.00 & 0.31$\pm$0.00 & 0.32$\pm$0.00 & 0.3$\pm$0.00 & 0.30$\pm$0.00 \\
       
        \midrule
        \multicolumn{10}{c}{\textbf{CE} ($\downarrow$)} \\ 
        \midrule
   Bibtex      & 0.23$\pm$0.05 & 0.29$\pm$0.00 & 0.14$\pm$0.03 & \textbf{0.09$\pm$0.04} & 0.25$\pm$0.01 & 0.16$\pm$0.05 & 0.18$\pm$0.07 & 0.24$\pm$0.04 & 0.31$\pm$0.01 \\
        Birds       & \textbf{0.22$\pm$0.01} & 0.57$\pm$0.03 & 0.47$\pm$0.05 & 0.57$\pm$0.03 & 0.87$\pm$0.00 & 0.45$\pm$0.01 & 0.68$\pm$0.08 & 0.57$\pm$0.09 & 0.62$\pm$0.04 \\
        CHD\_49     & 0.65$\pm$0.00 & 0.52$\pm$0.01 & 0.50$\pm$0.04 & \textbf{0.49$\pm$0.07} & 0.71$\pm$0.03 & 0.58$\pm$0.03 & 0.54$\pm$0.04 & 0.66$\pm$0.02 & 0.66$\pm$0.02 \\
        Chess       & 0.2$\pm$0.00 & 0.14$\pm$0.02 & \textbf{0.13$\pm$0.03} & 0.15$\pm$0.08 & 0.45$\pm$0.02 & 0.27$\pm$0.03 & 0.25$\pm$0.12 & 0.26$\pm$0.06 & 0.23$\pm$0.07 \\
        HumanPseAAC & \textbf{0.25$\pm$0.01} & 0.25$\pm$0.00 & 0.25$\pm$0.00 & 0.34$\pm$0.07 & 0.47$\pm$0.01 & 0.47$\pm$0.02 & 0.32$\pm$0.04 & 0.29$\pm$0.10 & 0.37$\pm$0.05 \\
        Mediamill   & \textbf{0.14$\pm$0.01} & \textbf 0.14$\pm$0.00 & \textbf 0.14$\pm$0.00 & 0.18$\pm$0.08 & 0.15$\pm$0.00 & 0.24$\pm$0.04 & 0.23$\pm$0.11 & 0.19$\pm$0.00 & 0.18$\pm$0.01 \\
        PlantPseAAC & 0.29$\pm$0.02 & 0.30$\pm$0.00 & 0.28$\pm$0.01 & \textbf{0.27$\pm$0.00} & 0.40$\pm$0.06 & 0.34$\pm$0.02 & 0.49$\pm$0.02 & 0.36$\pm$0.09 & 0.45$\pm$0.05 \\
        Slashdot    & 0.12$\pm$0.00 & 0.06$\pm$0.02 & \textbf{0.06$\pm$0.01} & 0.07$\pm$0.02 & 0.37$\pm$0.00 & 0.19$\pm$0.09 & 0.34$\pm$0.08 & 0.19$\pm$0.09 & 0.19$\pm$0.08 \\
        Yeast       & 0.54$\pm$0.01 & 0.52$\pm$0.00 & \textbf{0.50$\pm$0.03} & 0.58$\pm$0.05 & 0.53$\pm$0.04 & 0.54$\pm$0.01 & 0.59$\pm$0.04 & 0.56$\pm$0.05 & 0.78$\pm$0.00 \\
        \bottomrule
    \end{tabular}%
\end{table*}
\section{Implementation details.}
We use a Transformer encoder with $d_{\mathrm{model}}=128$, $n_{\mathrm{head}}=4$, and $n_{\mathrm{layer}}=2$.
The feed-forward dimension is $4d_{\mathrm{model}}$.
In Stage 1, the discriminator (encoder + label head) is optimized by Adam with learning rate $10^{-3}$,
while the policy head uses Adam with learning rate $10^{-4}$; the reward is defined by Stage 1.
The graph regularizer uses mini-batch $k$NN with $k=5$ , and we set
$\lambda_{\mathrm{struct}}=0.01$ unless otherwise specified.
In Stage 2, the horizon is $T=\min(k,d)$ and the feature-acquisition policy uses a softmax temperature
$\tau_{\mathrm{act}}=1.0$ and a moving-average baseline with momentum $\beta=0.9$.
\section{Budget Curves}
\begin{figure}[htbp]
    \centering
    \begin{subfigure}{0.1\textwidth}
        \centering
        \includegraphics[width=\textwidth]{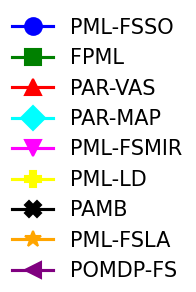}
    
    \end{subfigure}
    \begin{subfigure}{0.18\textwidth}
        \centering
        \includegraphics[width=\textwidth]{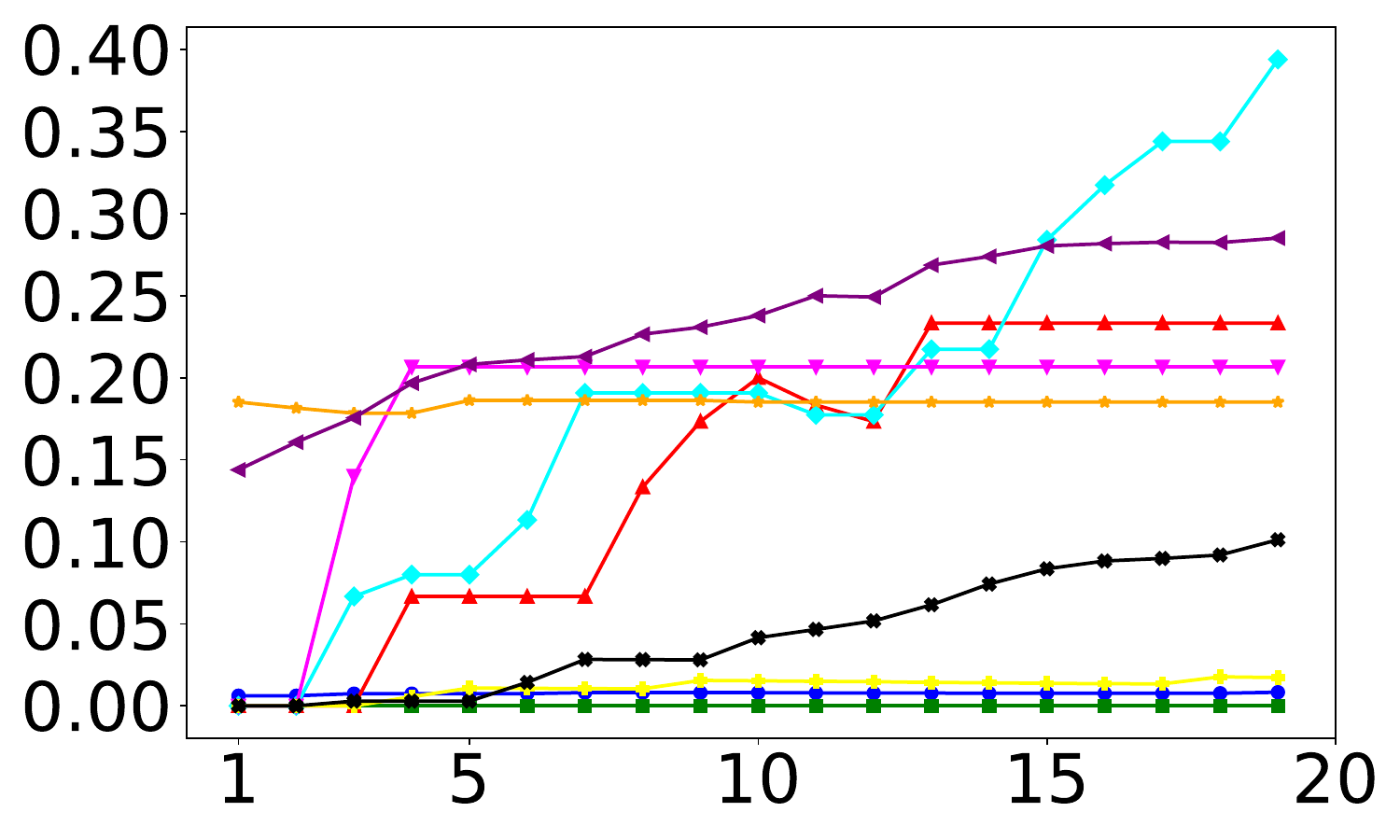}
        \caption{Micro-F1}
    \end{subfigure}
    \begin{subfigure}{0.18\textwidth}
        \centering
        \includegraphics[width=\textwidth]{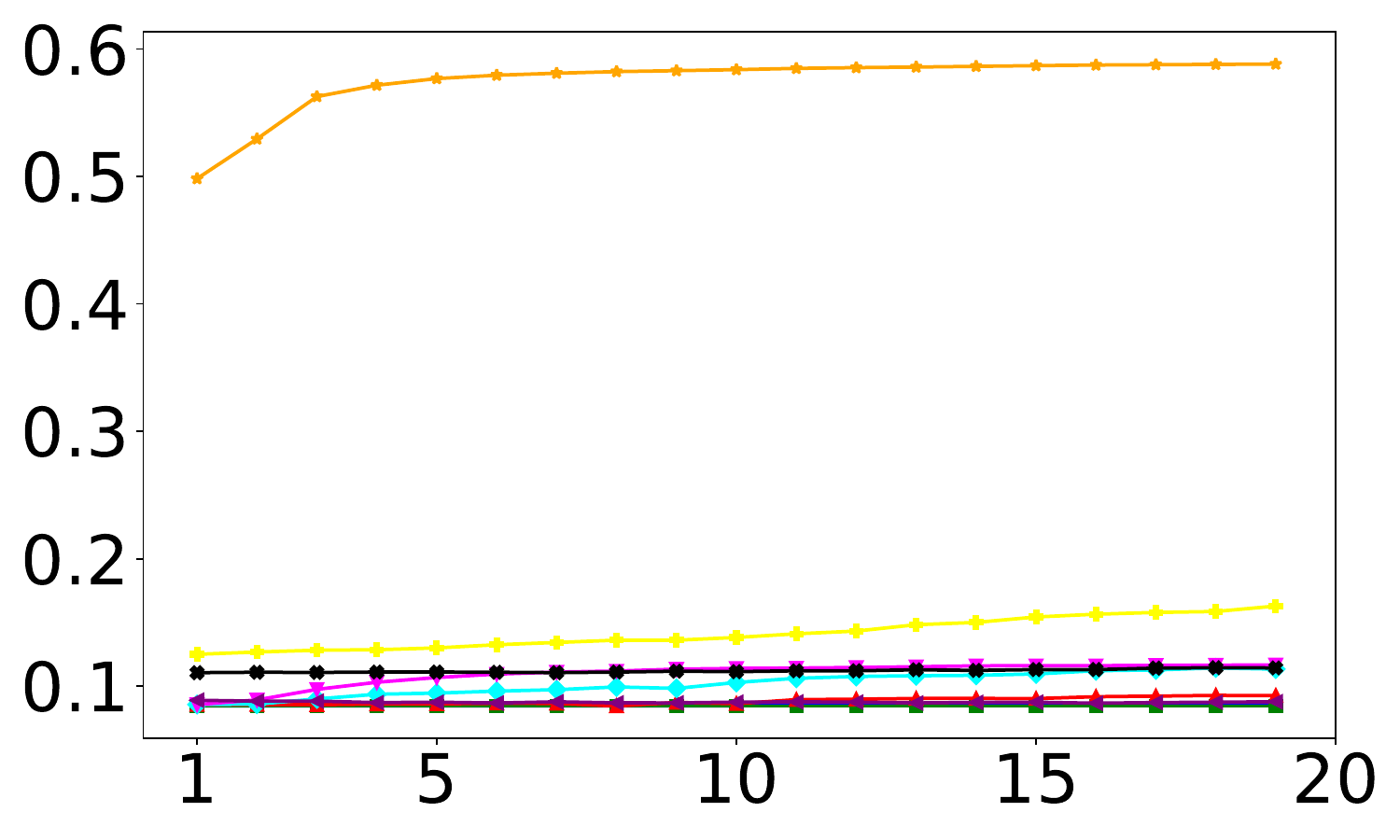}
        \caption{Hamming Loss}
    \end{subfigure}
    \begin{subfigure}{0.18\textwidth}
        \centering
        \includegraphics[width=\textwidth]{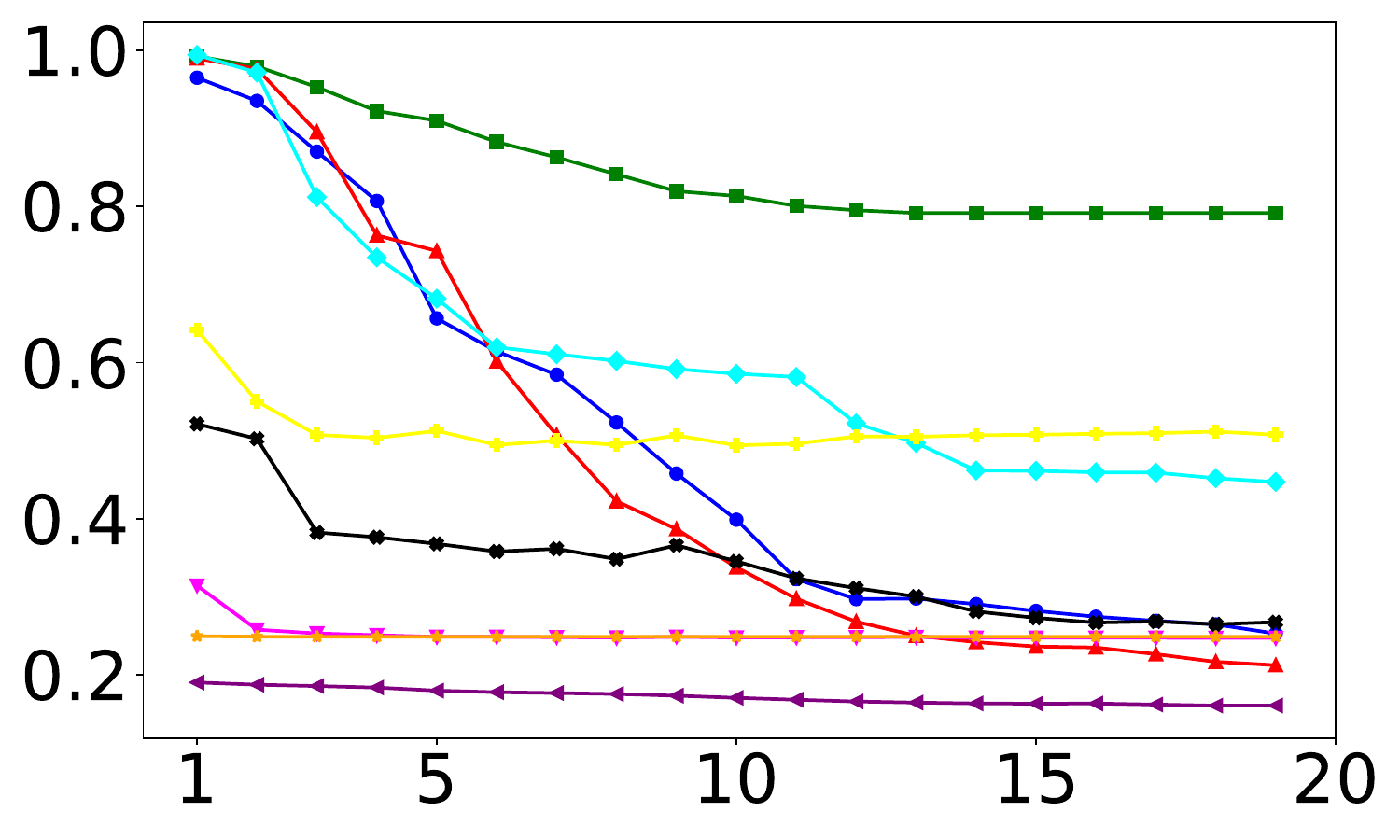}
        \caption{Ranking Loss}
    \end{subfigure}
    \begin{subfigure}{0.18\textwidth}
        \centering
        \includegraphics[width=\textwidth]{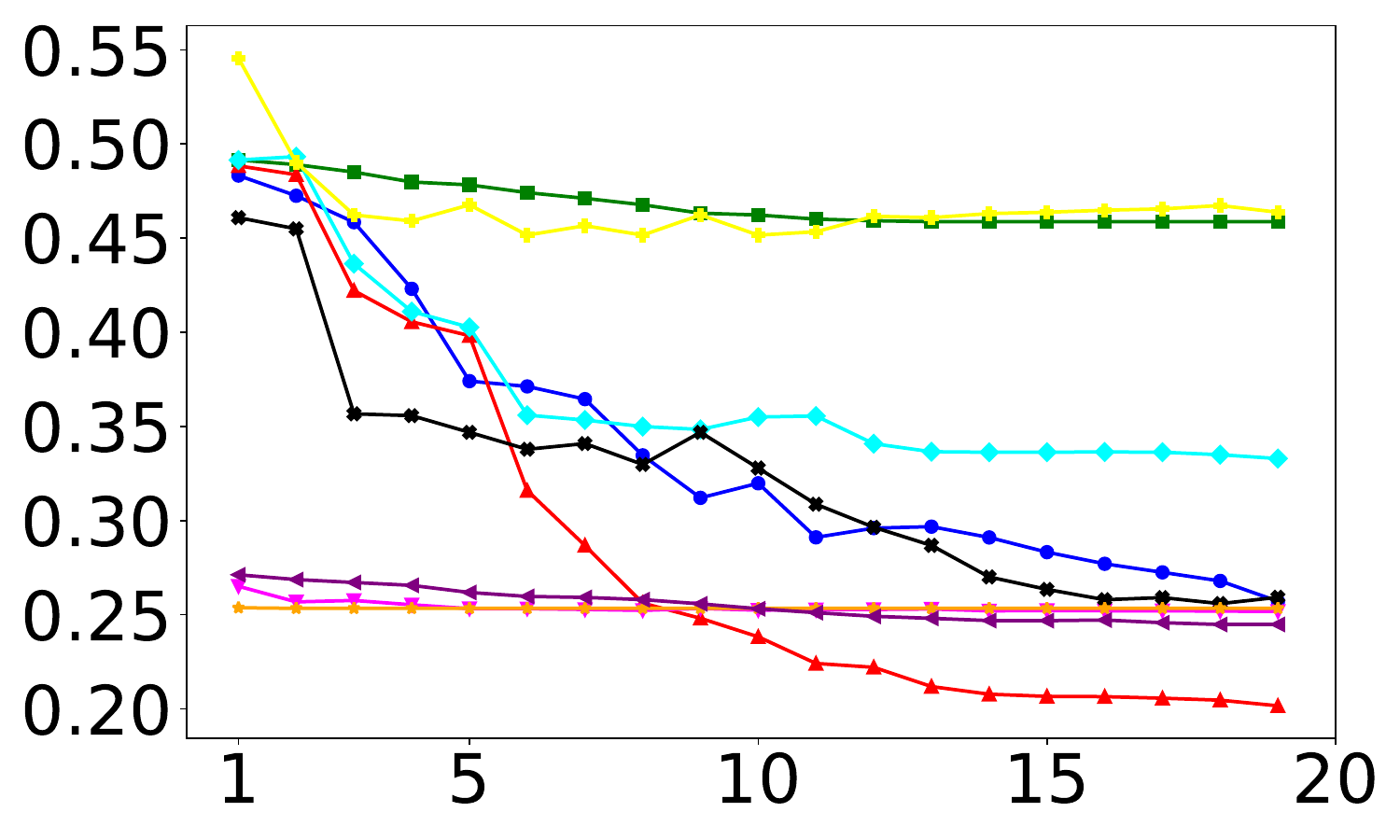}
        \caption{Coverage Error}
    \end{subfigure}
    \caption{ Nine methods on HumanPseAAC in terms of Micro-F1, Hamming Loss, Ranking Loss and Coverage Error.}
    
    \vspace{-0.7mm}
\end{figure}
\begin{figure}[htbp]
    \centering
    \begin{subfigure}{0.1\textwidth}
        \centering
        \includegraphics[width=\textwidth]{legend1.png}
    
    \end{subfigure}
    \begin{subfigure}{0.18\textwidth}
        \centering
        \includegraphics[width=\textwidth]{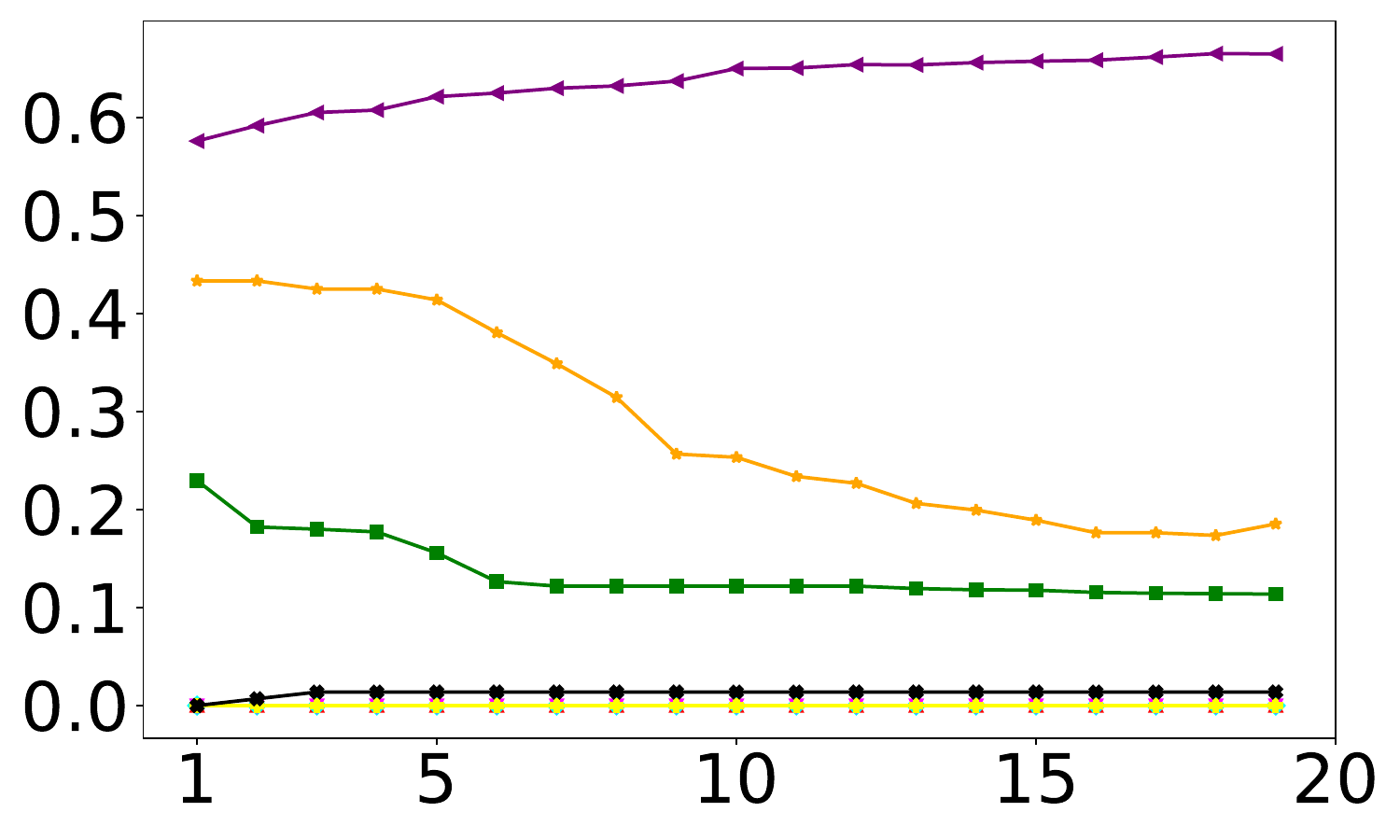}
        \caption{Micro-F1}
    \end{subfigure}
    \begin{subfigure}{0.18\textwidth}
        \centering
        \includegraphics[width=\textwidth]{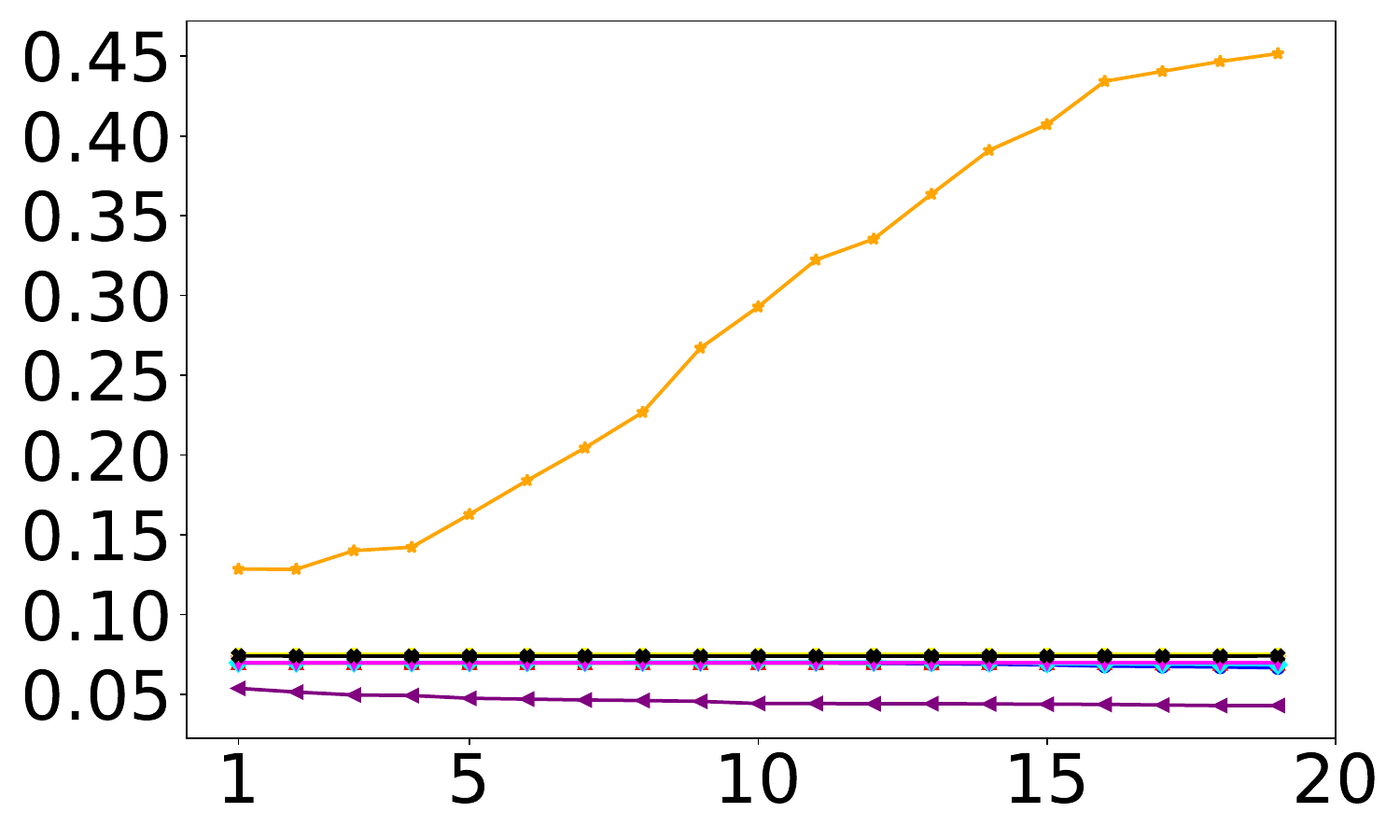}
        \caption{Hamming Loss}
    \end{subfigure}
    \begin{subfigure}{0.18\textwidth}
        \centering
        \includegraphics[width=\textwidth]{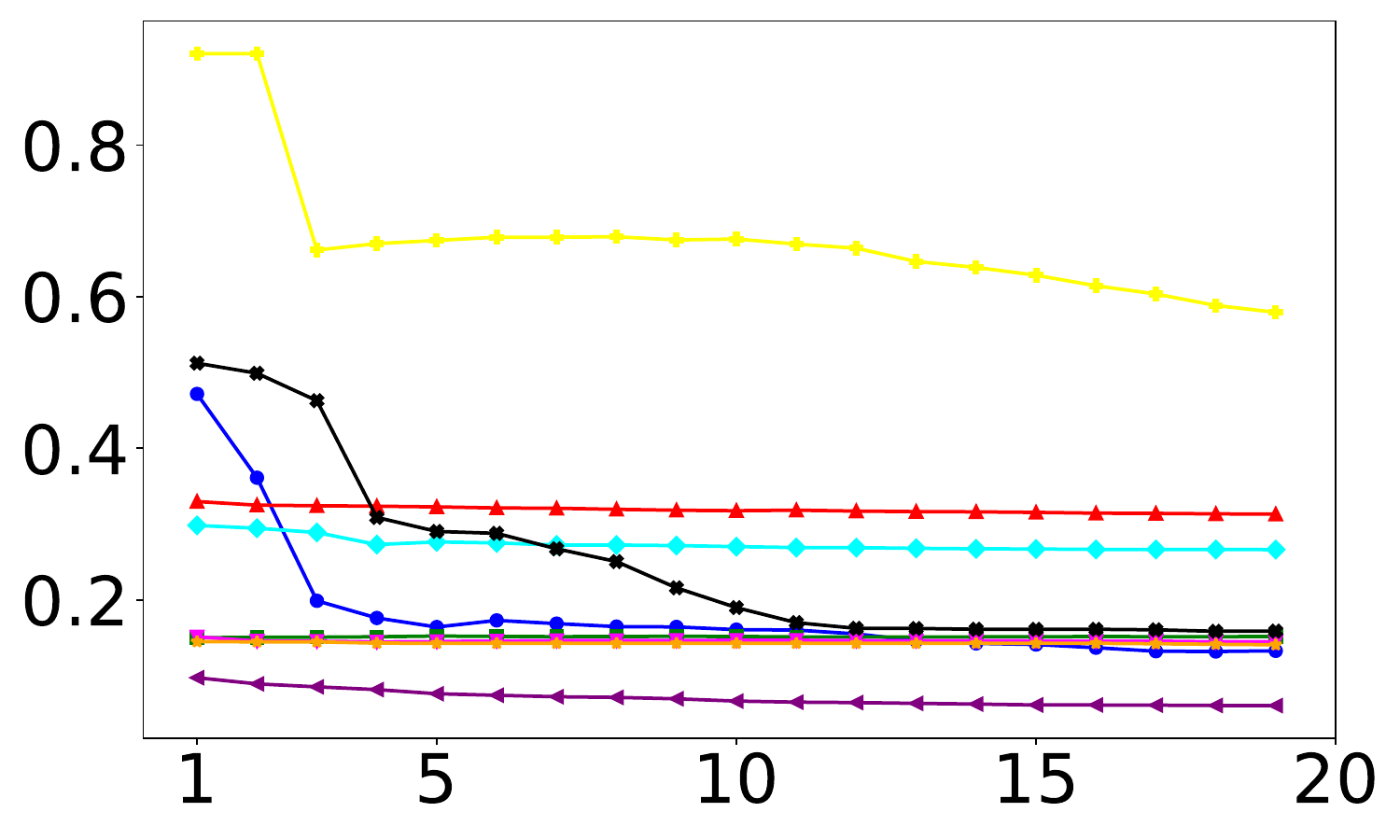}
        \caption{Ranking Loss}
    \end{subfigure}
    \begin{subfigure}{0.18\textwidth}
        \centering
        \includegraphics[width=\textwidth]{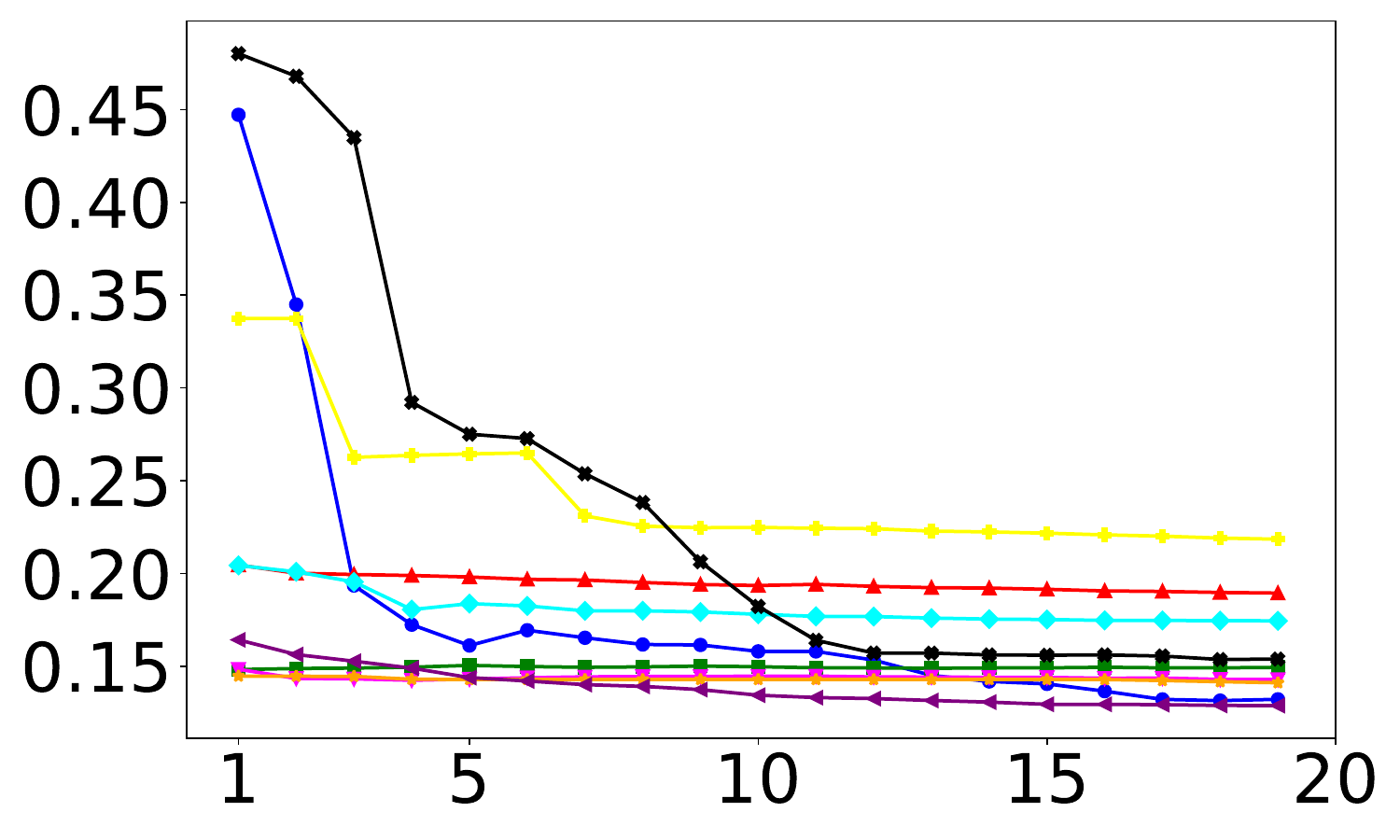}
        \caption{Coverage Error}
    \end{subfigure}
    \caption{ Nine methods on Mediamill in terms of Micro-F1, Hamming Loss, Ranking Loss and Coverage Error.}
    
    \vspace{-0.7mm}
\end{figure}


\end{document}